\documentclass[sigconf]{aamas} 

\usepackage{algorithmic}
\usepackage{algorithm}
\usepackage{subfigure}
\usepackage{balance} 
\usepackage{multicol}

\newtheorem{assumption}{Assumption}

\setcopyright{ifaamas}
\acmConference[AAMAS '24]{Proc.\@ of the 23rd International Conference
on Autonomous Agents and Multiagent Systems (AAMAS 2024)}{May 6 -- 10, 2024}
{Auckland, New Zealand}{N.~Alechina, V.~Dignum, M.~Dastani, J.S.~Sichman (eds.)}
\copyrightyear{2024}
\acmYear{2024}
\acmDOI{}
\acmPrice{}
\acmISBN{}

\acmSubmissionID{650}

\title[AAMAS-2024 Formatting Instructions]{Learning to Schedule Online Tasks with Bandit Feedback}

\author{Yongxin Xu}
\affiliation{
  \institution{Shanghaitech University}
  \city{Shanghai}
  \country{China}}
\email{xuyx2022@shanghaitech.edu.cn}

\author{Shangshang Wang}
\affiliation{
  \institution{ShanghaiTech University}
  \city{Shanghai}
  \country{China}}
\email{wangshsh2@shanghaitech.edu.cn}

\author{Hengquan Guo}
\affiliation{
  \institution{ShanghaiTech University}
  \city{Shanghai}
  \country{China}}
\email{guohq@shanghaitech.edu.cn}

\author{Xin	Liu}
\affiliation{
  \institution{ShanghaiTech University}
  \city{Shanghai}
  \country{China}}
\email{liuxin7@shanghaitech.edu.cn}

\author{Ziyu Shao}
\affiliation{
  \institution{ShanghaiTech University}
  \city{Shanghai}
  \country{China}}
\email{shaozy@shanghaitech.edu.cn}

\begin{abstract}
Online task scheduling serves an integral role for task-intensive applications in cloud computing and crowdsourcing.
Optimal scheduling can enhance system performance, typically measured by the reward-to-cost ratio, under some task arrival distribution.
On one hand, both reward and cost are dependent on task context (e.g., evaluation metric) and remain black-box in practice.
These render reward and cost hard to model thus unknown before decision making.
On the other hand, task arrival behaviors remain sensitive to factors like unpredictable system fluctuation whereby a prior estimation or the conventional assumption of arrival distribution (e.g., Poisson) may fail.
This implies another practical yet often neglected challenge, i.e., uncertain task arrival distribution.
Towards effective scheduling under a stationary environment with various uncertainties, we propose a double-optimistic learning based Robbins-Monro (DOL-RM) algorithm.
Specifically, DOL-RM integrates a learning module that incorporates optimistic estimation for reward-to-cost ratio and a decision module that utilizes the Robbins-Monro method to implicitly learn task arrival distribution while making scheduling decisions.
Theoretically, DOL-RM achieves 
a sub-linear regret of $O(T^{3/4})$, which is the first result for online task scheduling under uncertain task arrival distribution and unknown reward and cost.
Our numerical results in a synthetic experiment and a real-world application demonstrate the effectiveness of DOL-RM in achieving the best cumulative reward-to-cost ratio compared with other state-of-the-art baselines.
\end{abstract}


\newcommand{\BibTeX}{\rm B\kern-.05em{\sc i\kern-.025em b}\kern-.08em\TeX}

\makeatletter
\gdef\@copyrightpermission{
	\begin{minipage}{0.3\columnwidth}
		\href{https://creativecommons.org/licenses/by/4.0/}{\includegraphics[width=0.90\textwidth]{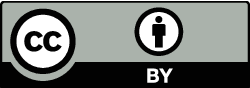}}
	\end{minipage}\hfill
	\begin{minipage}{0.7\columnwidth}
		\href{https://creativecommons.org/licenses/by/4.0/}{This work is licensed under a Creative Commons Attribution International 4.0 License.}
	\end{minipage}
	\vspace{5pt}
}
\makeatother

\begin{document}


\pagestyle{fancy}
\fancyhead{}


\maketitle 


\section{Introduction} \label{sec:intro}

\begin{figure}[!t]
    \centering
    \includegraphics[width=0.30\textwidth]{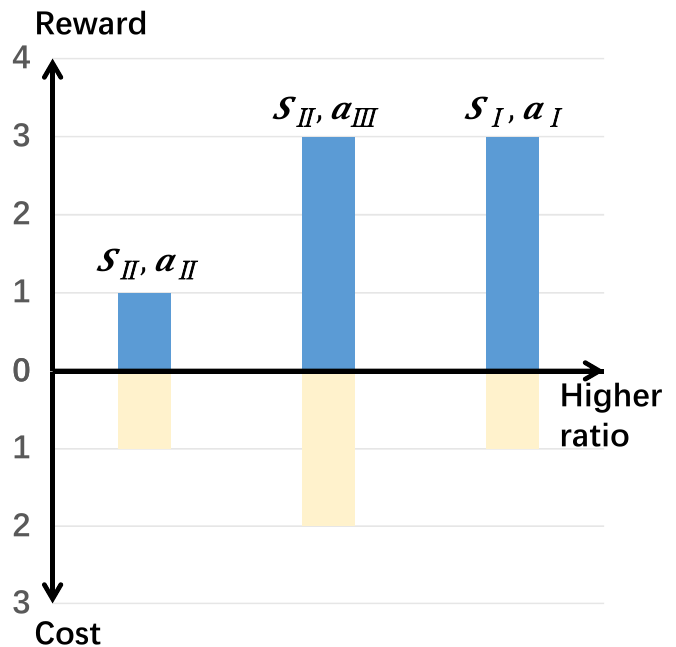}
    \vspace{- 0.1 cm}
    \caption{Rewards and Costs for $(S,a)$ in a Toy System.}
    \label{fig:Example}
\end{figure}
\begin{table}[!t]
    \centering
    \caption{Reward-to-Cost Ratio under Distribution $\{0.8,0.2\}$.}
    \begin{tabular}[!t]{ccc}
    \hline
    Algorithm & Policy & Ratio\\ 
    \hline
    The greedy algorithm & $(S_{\uppercase\expandafter{\romannumeral1}}, a_{\uppercase\expandafter{\romannumeral1}})$, $(S_{\uppercase\expandafter{\romannumeral2}}, a_{\uppercase\expandafter{\romannumeral3}})$ & 2.5\\
    The reverse algorithm & $(S_{\uppercase\expandafter{\romannumeral1}}, a_{\uppercase\expandafter{\romannumeral1}})$, $(S_{\uppercase\expandafter{\romannumeral2}}, a_{\uppercase\expandafter{\romannumeral2}})$
    & 2.6\\
    \hline
    \end{tabular}
    \label{tab:toy}
\end{table}

Online task scheduling refers to the process that schedules incoming tasks to available resources in real-time.
It plays a central role in boosting operational efficiency for cloud computing~\cite{DrASal_14, ARADhaVij_19, GuoZhaShe_12}, crowdsourcing~\cite{AfrMay_21, RomHarDan_11, DenCysZhu_15} and multiprocessing~\cite{Mac_96, SacVikGau_10, CenOrhIhs_10} systems, etc.
When a task is scheduled, it occupies resources for processing, e.g., computing hardware or labor costs.
Upon task completion, we receive rewards, e.g., a high-accuracy machine learning model or high-quality data.
The typical goal of online task scheduling is to optimize the {\it reward-to-cost ratio}, which represents the system efficiency. 
For instance, a cloud computing platform bills users to process their tasks using virtualized processing units, aiming to maximize the average revenue-to-hardware cost ratio.
A crowdsourcing platform compensates appropriate workers for processing data labeling tasks, aiming to optimize the quality-to-cost ratio.

There exist two major challenges to establish an effective online task scheduling policy.
The first one is the black-box nature of rewards and costs.
Take the machine learning model training as an example, rewards could refer to the model training accuracy, area under the curve (AUC) and \textit{f1} score while the costs could refer to the power consumption~\cite{MohMdN_15, FahLuJJav_15}.
The scheduler needs to choose an appropriate training time (decision) to maximize the accuracy (reward) per power consumption (cost).
However, the knowledge of accuracy and power consumption w.r.t. training time is unknown apriori and thus information is only revealed for the chosen decision (i.e., bandit feedback) when the task is completed.

The second challenge is the uncertainty of task arrival, where there exist multiple types of tasks in the system and the scheduler has no prior information on their arrival distribution.
One intuitive approach to addressing this challenge is always choosing the greedy decision with the maximal estimated ratio (suppose the knowledge of rewards and costs is known), regardless of the task type.
This idea seems to work because the greedy decision for every incoming tasks would maximize the average reward-to-cost across all tasks. 
However, we show the greedy decision could be an arbitrarily sub-optimal solution by slightly twisting the task arrival distribution. 
Consider a toy task scheduling system with arriving tasks of two types in~\figurename~\ref{fig:Example}, denoted as task $S_{\uppercase\expandafter{\romannumeral1}}$ and $S_{\uppercase\expandafter{\romannumeral2}}$.
For task $S_{\uppercase\expandafter{\romannumeral1}}$, there is only one available decision $a_{\uppercase\expandafter{\romannumeral1}}$.
For task $S_{\uppercase\expandafter{\romannumeral2}}$, two decisions exist, i.e., the first decision $a_{\uppercase\expandafter{\romannumeral2}}$ involves a low reward and a low cost, while the second one $a_{\uppercase\expandafter{\romannumeral3}}$ is characterized by a higher reward and a higher cost which has a higher reward-to-cost ratio than decision $a_{\uppercase\expandafter{\romannumeral2}}$.
We consider two algorithms, i.e., the greedy ($a_{\uppercase\expandafter{\romannumeral1}}$ for task $S_{\uppercase\expandafter{\romannumeral1}}$ and $a_{\uppercase\expandafter{\romannumeral3}}$ for task $S_{\uppercase\expandafter{\romannumeral2}}$) and the reverse ($a_{\uppercase\expandafter{\romannumeral1}}$ for task $S_{\uppercase\expandafter{\romannumeral1}}$ and $a_{\uppercase\expandafter{\romannumeral2}}$ for task $S_{\uppercase\expandafter{\romannumeral2}}$) algorithm.
One may argue that the greedy algorithm is more favorable since it achieves a higher ratio for task $S_{\uppercase\expandafter{\romannumeral2}}$.
However, when the task arrival distribution is $\{0.8, 0.2\}$ in~Table~\ref{tab:toy}, this intuitive algorithm achieves an expected ratio of $(0.8\times 3 + 0.2\times 3)/(0.8\times 1 + 0.2\times 2)=2.5$, which is worse than its reverse algorithm with that of $(0.8\times 3 + 0.2\times 1)/(0.8\times 1 + 0.2\times 1) = 2.6$.
One may then suggest to adopt the reverse algorithm instead.
However, we can always construct another distribution (e.g., $\{0.2,0.8\}$) to fail the reverse algorithm such that it turns out to be worse.
From this example, we emphasize the critical role of task arrival distribution in establishing effective online task scheduling algorithms. 
To address the two challenges above, we propose a double-optimistic learning approach to estimate the reward and cost with bandit feedback.
Intuitively, the learning approach establishes the optimistic and pessimistic estimations for rewards and costs, respectively, yielding an overall optimistic estimation for the reward-to-cost ratio.
This confirms the principle, the optimism in the face of uncertainty, which is exemplified by the confidence bound based algorithms~\cite{TorCsa_20}.
Furthermore, instead of adopting naive estimation for the task arrival distribution, we utilize the Robbins-Monro method to implicitly learn the task arrival distribution while making decisions.
Intuitively, this method transforms the problem of reward-to-cost ratio maximization into a fixed point problem, which can be efficiently solved by carefully using stochastic samples and yields a fast convergence.

In summary, we propose and analyze a novel optimization framework for online task scheduling, where the objective is to optimize the cumulative reward-to-cost ratio under a stationary environment with uncertain task arrival distribution.
We integrate double-optimistic learning and the Robbins-Monro method for effective and efficient online scheduling.
Our main contributions are summarized as follows:

\textbf{\textbullet \ Model for Online Task Scheduling.}
We propose a general framework for online task scheduling without any prior knowledge of reward, cost and task arrival distribution.
This framework encapsulates a variety of practical instances whose goal is to optimize the reward-to-cost ratio, thereby striving to achieve high-return and cost-effective outcomes in real-world scheduling systems.

\textbf{\textbullet \ Algorithm Design.}
We propose a novel algorithm called DOL-RM which incorporates double-optimistic learning for unknown rewards and costs and a modified Robbins-Monro method to implicitly learn the uncertain task arrival distribution.
This integrated design enables the balance between rewards and costs such that the cumulative reward-to-cost ratio is maximized.

\textbf{\textbullet \ Theoretical Analysis.}
We prove that DOL-RM achieves a sub-linear regret at order $O(T^{3/4})$ against the optimal scheduling policy in hindsight (Theorem \ref{thm:main}).
To prove the main result, we decompose the regret into two individual errors w.r.t. double-optimistic learning and Robbins-Monro method and then leverage the Lyapunov drift technique and carefully control the cumulative errors over the entire learning process.

\textbf{\textbullet \ Applications.} 
We test DOL-RM and compare it with state-of-the-art baselines via both a synthetic simulation and a real-world experiment of machine learning task scheduling.
These results demonstrate that DOL-RM can achieve the best cumulative reward-to-cost ratio without any prior knowledge of reward, costs and task arrival distribution.

\subsection{Related Work}

{\bf Online Task Scheduling.}
Online task scheduling has been studied extensively in the literature.
We focus on presenting the works most related to ours.
In~\cite{GaoMoTTay_21}, the authors utilize a UCB variant to address the exploration-exploitation dilemma in online task scheduling, however, the work does not take the cost into the consideration.
In~\cite{LiRMaQGon_21}, a Robbins-Monro based approach is proposed for task scheduling in edge computing.
However, its primary goal is to preserve data freshness, measured by Age of Information (AoI), differs from our setting of maximizing a generic reward-to-cost ratio.
In~\cite{SemAtiRay_19}, online bandit feedback is leveraged to schedule tasks in a renewal system, which is distinct from ours because it allows the scheduler to interrupt a task in service but the task in our model cannot be stopped once being scheduled. 
The closely related work is~\cite{Nee_21}, which studies a similar setting with this paper except assuming the decision set and the information of reward and cost are completely known before decision making. 
In contrast, we focus on a more practical scenario without assuming any knowledge of the rewards and costs.
Moreover, our algorithm and proof techniques are also different with~\cite{Nee_21} as we need to handle the additional uncertainties of the rewards and costs.

{\noindent\bf Cost-Aware Bandits:}
The key feature in our online task scheduling problem is that costs are concomitant with rewards upon decision-making.
This feature is also captured in another decision-making model, i.e., the cost-aware bandit model, where an agent incurs a reward and a cost simultaneously by pulling an arm.
The cost-aware bandits can be specialized into two major models, including budget-constrained bandits and the cost-aware cascading bandits.
The budget-constrained bandits have been widely explored in~\cite{SemAtiRSr_20, DebShwSuj_22, DinQinZha_13, LiHXia_17, LonArcEnr_10, LonArcAle_12, XiaDinZha_16, XiaQinDin_17, ZhoCla_18}, where pulling an arm incurs an additional cost and there exists a hard stopping point once the cumulative cost exceeds a given budget. 
Our setting has two major differences with the budget-constrained bandits.
The first one is we do not have or assume any explicit budget limit constraints; and the second one is we consider a ``continual'' system, where it would not stop in the middle until all tasks are completed.
Therefore, the previous algorithms and analysis in budget-constrained bandits cannot be applied in this paper. 
Besides, cost-aware cascading bandits are also related and have been studied in~\cite{CheHuaShe_22, GanZhoYan_18, GanZhoYan_20, AniGauLiS_22, WanZhoYan_19}.
Its goal is to maximize the reward-to-cost gap, however, such a maximal gap does not necessarily yield the optimal ratio as in our paper because we have multiple types of tasks in the system and need to take the task arrival distribution into consideration.

Compared to the previous works, this paper takes a bold step and investigates online task scheduling without any knowledge of rewards, costs, and task arrival distribution. 
Accordingly, we propose DOL-RM, an effective algorithm for optimizing the reward-to-cost ratio in online task scheduling. 
Furthermore, DOL-RM establishes a near-optimal regret performance for this challenging setting and has been validated to achieve the best empirical performance via a synthetic simulation and a real-world experiment.

\section{System Model}

We study a typical online task scheduling system where a central controller processes an incoming sequence of heterogeneous tasks in a back-to-back manner (i.e., the controller observes and processes the next task immediately upon the completion of its current task).
We assume $S$ types of incoming tasks in the system, denoted by the set $\mathcal S \triangleq \{1,2,\cdots, S\}$ and every task is randomly drawn from a stationary probability distribution $\mathbb P: \mathcal S \to \mathbb R^{+}$. 
For the $t$th task, the controller observes its type $S_t \in \mathcal S$ and chooses a decision $a_t \in \mathcal A(S_t),$ where $\mathcal A(S_t)$ is the corresponding available decision set of task type $S_t$.
When the $t$th task is completed, the controller observes the feedback of reward $R_{S_t,a_t}$ and cost $C_{S_t,a_t},$ which are randomly generated from with the expected values $r_{S_t,a_t} \triangleq \mathbb E[R_{S_t,a_t}]$ and $c_{S_t,a_t} \triangleq C_{S_t,a_t}.$ 
Note we only observe the feedback with respect to the selected decision $a_t,$ thus coined as ``bandit feedback'', a term borrowed from bandit learning~\cite{TorCsa_20}.
Then, the controller continues to work on the next $(t+1)$th task until $T$ tasks are completed.
In the paper, we assume the rewards $\{R_{S_t,a_t}\}_t$ and costs $\{C_{S_t,a_t}\}_t$ are independent if $(S_t, a_t)$ is the same. 

\begin{figure}[!t]
    \centerline{\includegraphics[width=0.43\textwidth]{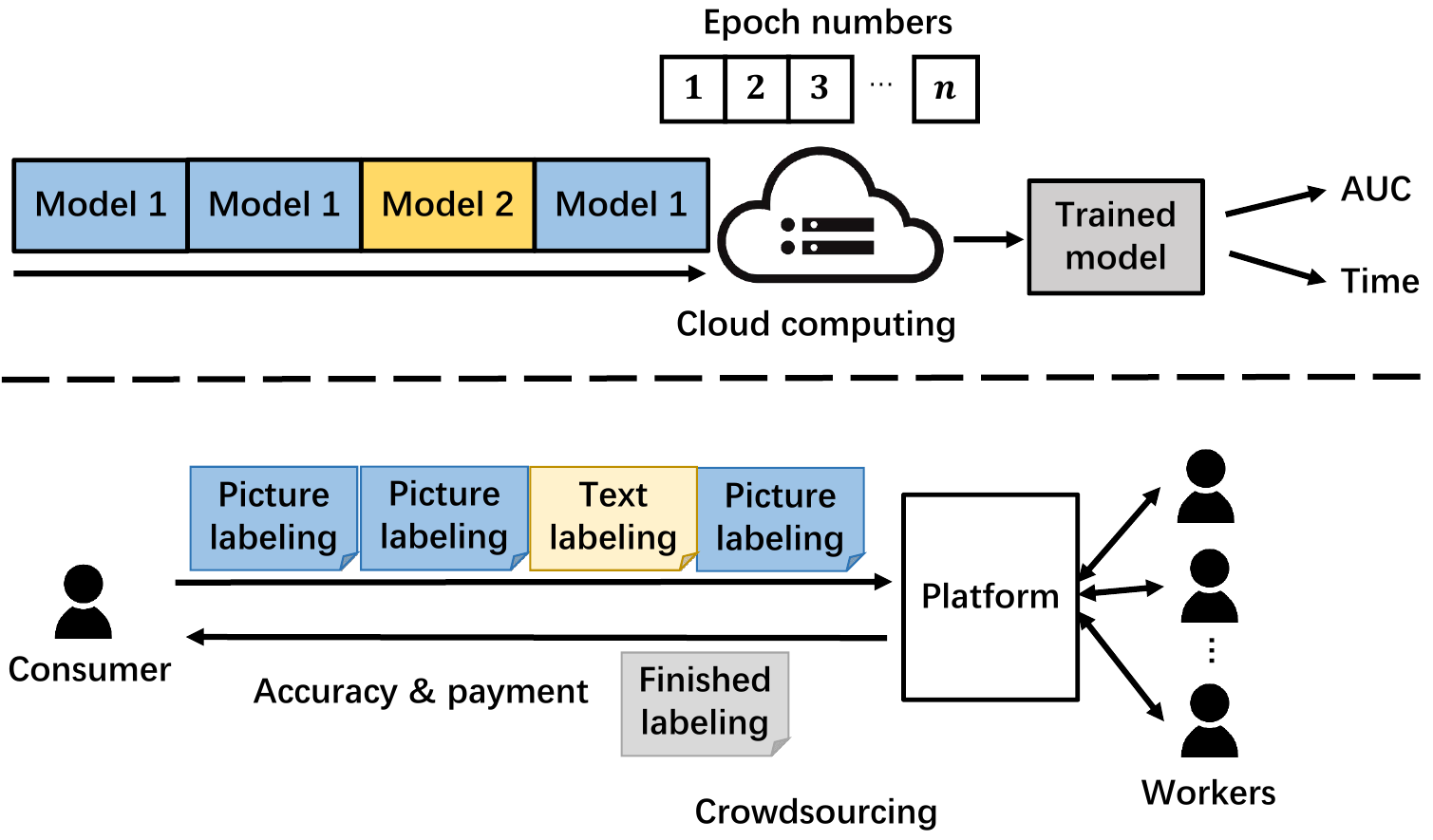}}
    \caption{Online Task Scheduling in Cloud Computing (above) and Crowdsourcing (bottom).}
    \label{fig:Model}
\end{figure}

{\bf Reward-to-Cost Ratio Maximization:}
 The controller aims to maximize the cumulative \textit{reward-to-cost ratio} over a sequence of $T$ tasks as follows:
\begin{equation}
    \label{eq:offline}
    \max_{ \{ a_{t} \}_{t} } \frac{\sum_{t=1}^{T}~\mathbb{E}[R_{S_t,a_t}]}{\sum_{t=1}^{T}~\mathbb{E}[C_{S_t,a_t}]}.
\end{equation}
Expectation in Problem~\eqref{eq:offline} is taken over the randomness w.r.t. reward, cost and task arrival distribution $\mathbb{P}(\cdot)$.
We emphasize that we do not assume any prior knowledge on such distribution and instead learn it implicitly in an online manner (please refer to Section~\ref{sec:algo}).

When the task arrival distribution and the expected reward and cost are known, we can compute the optimal policy by solving Problem~\eqref{eq:offline} with offline optimization techniques~\cite{Ben_66}.
However, such knowledge is practically unavailable or infeasible to the controller and we have to learn it by exploring the decision space in an online manner.
Moreover, as discussed in the introduction, the greedy algorithm that simply maximizes ${R_{S_t,a_t}}/{C_{S_t,a_t}}$ for each individual task can achieve arbitrary sub-optimal performance because it ignores the task arrival distribution.
Therefore, we propose an online learning based algorithm to address challenges on the uncertainties of rewards, costs and task arrivals.

Before presenting our algorithm, we emphasize that our modeling is general to capture many real-world applications, 
which is illustrated with two following examples as shown in~\figurename~\ref{fig:Model}: 
\begin{itemize}
    \item Machine Learning (ML) model training on cloud servers:
    The cloud platforms (e.g., AWS, Microsoft Azure or Google Cloud) constantly process ML model training tasks submitted by users.
    The types of ML training tasks are unknown apriori and only revealed at the server when they are executed. 
    For each upcoming task, the platform decides a specific training time from a range of options, which corresponds to the cost; when the training process is finished, the test accuracy is returned as the reward.
    However, the relationship of accuracy v.s. training time is uncertain. 
    The goal of the platform is to schedule the training time for each individual task such that the average accuracy per time unit (i.e., accuracy-to-time ratio) is maximized.
    \item Tasks assignment in Crowdsourcing:
    The crowdsourcing platforms (e.g., Amazon Turk or Task Rabbit) received a sequence of data labeling tasks submitted by consumers.
    The platform is unaware of the type of tasks until they arrive. 
    For each task, the platform assigns it to a suitable worker from a set of candidates with various task proficiency and expertise.
    The quality of the labeled data from workers is considered as the reward, while his/her payment corresponds to the cost.
    The goal of the platform is to assign the worker for each individual task such that the average quality per dollar (i.e., quality-to-payment ratio) is maximized. 
\end{itemize}
Next, we present our online learning and decision algorithm to solve the reward-to-cost ratio maximization in Problem~\eqref{eq:offline}.

\section{Algorithm Design}
\label{sec:algo}

In this section, we propose a double-optimistic learning based Robbins-Monro (DOL-RM) algorithm to address the challenges induced by unknown reward, cost and task arrival distributions. 
Specifically, DOL-RM includes two algorithmic modules, i.e., the \textit{learning module} leverages double-optimistic learning for unknown rewards and costs, and the \textit{decision module} utilizes the Robbins-Monro method to make decisions via implicitly learning the task arrival distribution.
As shown in \figurename~\ref{fig:algorithm flow}, upon each task arrives, the learning module first estimates the reward and cost via a double-optimistic learning approach.
Based on the estimation, the decision module makes a proper decision through the Robbins-Monro method, which in turn provides an iteratively approaching estimation for the reward-to-cost ratio even without explicitly learning the task arrival distribution.
Next, we introduce these two major modules in detail and explain the intuition behind the algorithm.
The completed DOL-RM algorithm is depicted in \textbf{Algorithm}~\ref{alg:DOL-RM}.
\begin{figure}[!t]
    \centerline{\includegraphics[width=0.31\textwidth]{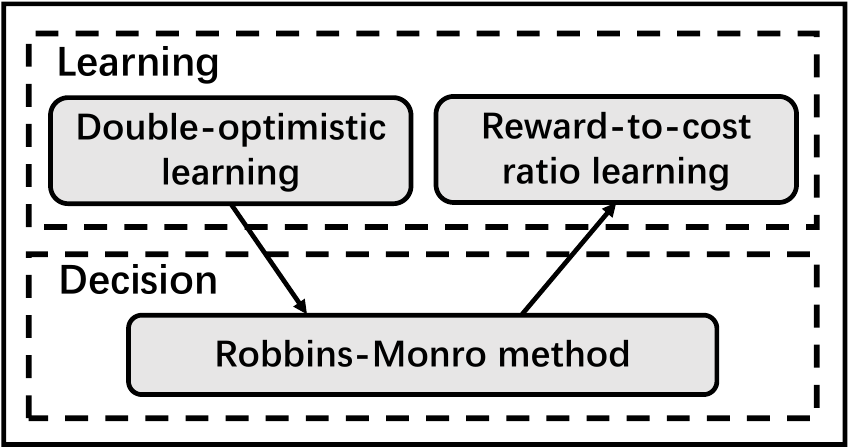}}
    \caption{Illustration of DOL-RM Algorithm.}
    \label{fig:algorithm flow}
\end{figure}

{\bf Double-Optimistic Learning:}
Recall that for the $t$th task with the type $S_t,$ the expected reward  $r_{S_{t}, a}$ and cost $c_{S_{t}, a}$ for each decision $a \in \mathcal A(S_t)$ are unknown.
We consider them as multi-armed bandit problems, where each decision is regarded as an arm associated with a pair of reward and cost. 
Accordingly, we leverage the canonical algorithm design principle in bandit learning~\cite{TorCsa_20}, optimism in the face of uncertainty, to estimate the reward and cost as follows:
\begin{align}
    \hat{r}_{S_t,a} \triangleq \min \left\{r_{\max}, \bar{R}_{S_t,a} + \sqrt{\frac{\log T}{N_{S_t,a}(t-1)}} \right\}, \label{eq:reward estimator} \\ 
    \check{c}_{S_t,a} \triangleq \max \left\{ c_{\min},\bar{C}_{S_t,a} - \sqrt{\frac{\log T}{N_{S_t, a}(t-1)}} \right\}, \label{eq:cost estimator}
\end{align}
where $\bar{R}_{S_t,a}$ and $\bar{C}_{S_t,a}$ denote the empirical means of reward and cost; $N_{S_t, a}(t-1)$ denotes times the arm (decision) $a$ has been chosen for task of type $S_t$ until round $(t-1)$; $r_{\max}$ and $c_{\min}$ are the maximum reward and minimum cost, respectively.
Note that $ \hat{r}_{S_t,a} $ and $ \check{c}_{S_t,a} $ denote the truncated Upper Confidence Bound (UCB) and Lower Confidence Bound (LCB) for reward and cost estimation, respectively.
Recall in Problem~\eqref{eq:offline}, $ \hat{r}_{S_t,a} $ and $ 1 / \check{c}_{S_t,a} $ serve as the optimistic estimators for $\mathbb{E} [R_{S_t, a}]$ and $1/ \mathbb{E} [C_{S_t, a}]$, yielding the name of ``Double-Optimistic Learning''.
Intuitively, this learning policy encourages exploration among arms (decision space) to learn the optimal reward-to-cost ratio effectively.

Suppose we make decision $a_t$ for the task $S_t$ and observe the reward $R_{S_t,a_t}$ and cost $C_{S_t,a_t}$ feedback when the task is completed, we conduct the following updates:
\begin{align}
    N_{S_t, a_t}(t) & = N_{S_t, a_t}(t-1) + 1, \label{eq:counter update}\\
    \bar{R}_{S_{t+1},a_t} & = \frac{\bar{R}_{S_t,a_t}N_{S_t, a_t}(t-1) + R_{S_t, a_t}}{N_{S_t, a_t}(t)}, \label{eq:ucb update}\\
    \bar{C}_{S_{t+1},a_t} & = \frac{\bar{C}_{S_t,a_t}N_{S_t, a_t}(t-1) + C_{S_t, a_t}}{N_{S_t, a_t}(t)}. \label{eq:lcb update}
\end{align}

\begin{algorithm}[!t]
    \caption{DOL-RM}
    \begin{algorithmic}
    \STATE
    \textbf{Input:} Number of tasks $T$, minimum cost $c_{\min}$, $\theta_{\max}$ and $\theta_{\min}$.
    \STATE
    \textbf{Initialization:}
    $\theta_1 \gets \theta_{\min}$, $\eta \gets \frac{1}{c_{\min}\sqrt{T}}$.
    \FOR {$t=1, 2, \dots, T$}
        \STATE
        \textbf{\textbullet Observation:} Task type $S_t$ and available decision set $\mathcal{A}(S_t)$.
        \STATE
        \textbf{\textbullet Double-Optimistic Learning:} 
        Calculate reward and cost estimates $\hat{r}_{S_t,a}$ and $\check{c}_{S_t,a}, \forall a\in\mathcal{A}(S_t)$ according to~\eqref{eq:reward estimator} and~\eqref{eq:cost estimator}, respectively.
        \STATE
        \textbf{\textbullet Robbins-Monro Based Decision:}
        Choose the greedy decision for task scheduling:
	  \begin{align}
              a_t \gets \underset{a \in \mathcal{A}(S_t)}{\arg\max} \ \ \hat{r}_{S_t,a} - \theta_t \ \check{c}_{S_t,a} \label{eq:decision}
	  \end{align}
        \STATE
        \textbf{\textbullet Bandit Feedback:}
        Obtain reward $R_{S_t, a_t}$ and cost $C_{S_t, a_t}$.
	  \STATE
        \textbf{\textbullet Reward-to-Cost Ratio Learning:}
        \begin{align}
            \theta_{t+1} \gets \left[ \theta_t + \eta (\hat{r}_{S_t,a_t} - \theta_t \ \check{c}_{S_t,a_t}) \right]_{\theta_{\min}}^{\theta_{\max}}. \label{eq:vqupdate}
        \end{align}
        Update the empirical means of rewards and costs with~\eqref{eq:counter update}--\eqref{eq:lcb update}.
    \ENDFOR
    \end{algorithmic}
    \label{alg:DOL-RM}
\end{algorithm}

{\bf Robbins-Monro Based Decision:}
Let $\theta^*$ be the optimal objective (i.e., reward-to-cost ratio) in Problem~\eqref{eq:offline}.
It is notable to observe that Problem~\eqref{eq:offline} is equivalent to finding an optimal sequence of actions such that:
\begin{equation}
    \max_{\{a_{t}\}_{t}} ~\mathbb{E}\left[{\sum_{t=1}^{T} R_{S_t,a_t}} - \theta^*{\sum_{t=1}^{T}C_{S_t,a_t}}\right] = 0, \label{eq:intuition}
\end{equation}
This is essentially a fixed point problem and can be solved using the Robbins-Monro method with stochastic samples~\cite{Nee_13, HerSut_51}.
Specifically, assuming the optimal ratio $\theta^*$ is known and using $\hat r_{S_t,a}$ and $\check c_{S_t,a}$ as the proxy of $R_{S_t,a_t}$ and $C_{S_t,a_t}$, we can follow a greedy decision for task $S_t \in \mathcal{S}$ according to~\eqref{eq:intuition}:
\begin{equation*}
    \underset{a \in \mathcal A(S_t)}{\arg\max} ~{\hat r_{S_t,a}} -\theta^* \ \check c_{S_t,a}
\end{equation*}
However, this decision is non-causal and infeasible due to the lack of knowledge on the best ratio $\theta^*$.
To learn $\theta^*$, we modify the Robbins-Monro iteration method by plugging the $\hat r_{S_t,a}$ and $\check c_{S_t,a}:$
\begin{equation*}
    \theta_{t+1} = \left[ \theta_t + \eta (\hat{r}_{S_t,a_t} - \theta_t \check{c}_{S_t,a_t}) \right]_{\theta_{\min}}^{\theta_{\max}}
\end{equation*}
where $\eta$ is the learning rate and $[\cdot]_{\theta_{\min}}^{\theta_{\max}}$ denotes the projection of real number onto the interval $[\theta_{\min}, \theta_{\max}]$ with $\theta_{\min} \triangleq {r_{\min}}/{c_{\max}}$ and $\theta_{\max} \triangleq {r_{\max}}/{c_{\min}}$.

Note this is different with the classical RM method~\cite{HerSut_51, Nee_21} which adopts the samples $R_{S_t,a_t}$ and $C_{S_t,a_t}$ for updating $\theta_t$.
Finally, our Robbins-Monro based decision is defined as follows:
\begin{equation*}
    \underset{a \in \mathcal A(S_t)}{\arg\max} ~{\hat r_{S_t,a}} -\theta_t \ \check c_{S_t,a}. 
\end{equation*}
where $\theta_t$ is treated as an estimation of $\theta^*$. 

This method circumvents the need to directly estimate the task arrival distribution.
Instead, it learns the optimal cumulative ratio by iteratively adjusting $\theta$ with the term $\eta (\hat{r}_{S_t,a_t} - \theta_t \check{c}_{S_t,a_t})$.
When $\theta$ is excessively large, $\eta (\hat{r}_{S_t,a_t} - \theta_t \check{c}_{S_t,a_t})$ tends to be negative, driving a downward shift, whereas a small $\theta$ prompts its value to skew positive, inducing an upward adjustment in itself.
With a proper learning rate $\eta$, $\theta$ can steadily converge to the optimal balance ratio as iterations increase (please refer to Section~\ref{sec:theo}).

\section{Theoretical Results} \label{sec:theo}

To present our main results, we first introduce the common assumptions on rewards and costs. 
\begin{assumption} \label{assumption:reward}
    The reward $R_{S_t,a}$ is a sub-Gaussian random variable with mean $ r_{S_t,a} = \mathbb E[R_{S_t,a}] \in [r_{\min}, r_{\max}] $ for any $S_t \in \mathcal S,$ $a \in \mathcal A(S_t)$ and $t \in [T].$  
\end{assumption}
\begin{assumption}\label{assumption:cost}
    The cost $C_{S_t,a}$ is a positive sub-Gaussian random variable with mean $ c_{S_t,a} = \mathbb E[C_{S_t,a}] \in [c_{\min}, c_{\max}]$ for any $S_t \in \mathcal S,$ $a \in \mathcal A(S_t)$ and $t \in [T].$  
\end{assumption}

{\noindent \bf Regret \& Convergence Gap:}
Let $\theta^*$ be the optimal reward-to-cost ratio in Problem~\eqref{eq:offline}, we define $Gap(T)$ to be the \textit{convergence gap}: 
\begin{equation}
     Gap(T) \triangleq \left\vert\theta^*-\frac{\sum_{t=1}^T\mathbb{E}[R_{S_t,a_t}]}{\sum_{t=1}^T\mathbb{E}[C_{S_t,a_t}]}\right\vert, \label{eq:gap}
\end{equation}
which measures the distance between the cumulative reward-to-cost ratio returned by a policy and the optimal ratio.
We define the \textit{regret} to be:
\begin{equation*}
    \mathcal{R}(T) \triangleq T \cdot Gap(T).
\end{equation*}
Our goal is to show that DOL-RM achieves sub-linear regret $o(T)$, i.e., $\lim_{T \rightarrow \infty} \frac{ \mathcal{R}(T) }{T} = \lim_{T \rightarrow \infty} Gap(T) = 0$, which implies that DOL-RM converges to the optimal policy and achieves the best reward-to-cost ratio in the long-term.
We state our main result for DOL-RM in the following theorem.

\begin{theorem}\label{thm:main}

Suppose Assumption~\ref{assumption:reward} and~\ref{assumption:cost} hold, DOL-RM achieves the following regret:
\begin{equation*}
    \mathcal{R}(T) = O(T^{\frac{3}{4}}).
\end{equation*}

\end{theorem}

Theorem~\ref{thm:main} highlights DOL-RM's favorable theoretical performance in achieving \textit{no regret learning} with sub-linear regrets $O(T^{\frac{3}{4}}),$ which implies the convergence gaps of $O(T^{-\frac{1}{4}})$.
To the best of our knowledge, these are the first results for online task scheduling problems without any prior information on rewards,  costs and task arrival distributions.
These results also indicate that DOL-RM can quickly identify an effective and efficient policy that converges to the optimal ratio $\theta^*$ with the integral design of double-optimistic learning and the Robbins-Monro method. 

We want to further mention two related works cited as~\cite{SutZhaYan_21} and~\cite{Nee_21}.
\cite{SutZhaYan_21} studied the reward-to-cost ratio in a Markov decision process (MDP).
Though MDP includes bandit as a special case, \cite{SutZhaYan_21} only established an asymmetrical convergence.
\cite{Nee_21} assumed perfect information on the rewards and costs and established improved regrets of $O(\sqrt{T})$.
With such perfect information, \cite{Nee_21} only needs to quantify the uncertainty of task arrival without any bias\footnote{Rewards and costs are usually biased with noises, e.g., sub-Gaussian noise in our case, while~\cite{Nee_21} assumes constant rewards and costs, thus no bias involved.} from rewards and costs.
This enables \cite{Nee_21} to conduct an aggressive learning scheme to control the bias only from the task arrival such that it can provide an anytime gap performance and improved performance.
However, when the rewards and costs are unknown, we must carefully balance all these uncertainties and control the bias over the entire time horizon (Lemma~\ref{lem:estimatebound} and~\ref{lem:rmbound}), which requires a conservative learning scheme and advanced Lyapunov drift techniques to achieve sub-linear gap given coupled uncertainties.


Next, we present the detailed proof of Theorem~\ref{thm:main} and focus on the analysis of convergence gap $Gap(T)$, where we first decompose $Gap(T)$ into the two items related to the double-optimistic learning and Robbins-Monro iteration method, respectively, and then established the items individually.
 
\subsection{Proof of Theorem~\ref{thm:main}}
Recall the definition of convergence gap in \eqref{eq:gap}.
We first decompose the convergence gap using triangle inequality by involving $\frac{\sum_{t=1}^T\mathbb{E}[\hat{r}_{S_t,a_t}]}{\sum_{t=1}^T\mathbb{E}[\check{c}_{S_t,a_t}]}$ as follows
\begin{align}
    Gap(T) \leq& \underbrace{\left\vert\frac{\sum_{t=1}^T\mathbb{E}[\hat{r}_{S_t,a_t}]}{\sum_{t=1}^T\mathbb{E}[\check{c}_{S_t,a_t}]}-\frac{\sum_{t=1}^T\mathbb{E}[r_{S_t,a_t}]}{\sum_{t=1}^T\mathbb{E}[c_{S_t,a_t}]}\right\vert}_\text{Double-optimistic learning error}.\label{eq:estimate-real bound} \\ 
    &+ \underbrace{\left\vert\theta^*-\frac{\sum_{t=1}^T\mathbb{E}[\hat{r}_{S_t,a_t}]}{\sum_{t=1}^T\mathbb{E}[\check{c}_{S_t,a_t}]}\right\vert.}_\text{Robbins-Monro iteration convergence gap}\label{eq:optimal-estimate gap}
\end{align}
For the first term~\eqref{eq:estimate-real bound}, we establish the bound in Lemma~\ref{lem:estimatebound}, which is related to the bias of optimistic learning on rewards and costs.
For the second term~\eqref{eq:optimal-estimate gap}, we establish the bound in Lemma~\ref{lem:rmbound} by quantifying the cumulative bias from the Robbins-Monro iteration.
\begin{lemma}\label{lem:estimatebound}
Suppose Assumption~\ref{assumption:reward} and~\ref{assumption:cost} hold, DOL-RM achieves
\begin{equation*}
    \left\vert\frac{\sum_{t=1}^T\mathbb{E}[\hat{r}_{S_t,a_t}]}{\sum_{t=1}^T\mathbb{E}[\check{c}_{S_t,a_t}]}-\frac{\sum_{t=1}^T\mathbb{E}[r_{S_t,a_t}]}{\sum_{t=1}^T\mathbb{E}[c_{S_t,a_t}]}\right\vert= O(T^{-\frac{1}{2}}). 
    \end{equation*}
\end{lemma}
\begin{lemma} \label{lem:rmbound}
Suppose Assumption~\ref{assumption:reward} and~\ref{assumption:cost} hold, DOL-RM achieves
    \begin{equation*}
        \left\vert\theta^*-\frac{\sum_{t=1}^T\mathbb{E}[\hat{r}_{S_t,a_t}]}{\sum_{t=1}^T\mathbb{E}[\check{c}_{S_t,a_t}]}\right\vert = O(T^{-\frac{1}{4}}).
    \end{equation*}
\end{lemma}
Based on these two lemmas, we prove the convergence gap in Theorem~\ref{thm:main} as follows:
\begin{equation*}
    Gap(T) \leq O(T^{-\frac{1}{2}}) + O(T^{-\frac{1}{4}}) = O(T^{-\frac{1}{4}}).
\end{equation*}
Here we only offer order-wise results in $O(\cdot)$ for the sake of presentation and the detailed expressions are delegated to the Appendix. 

Next, we prove Lemma~\ref{lem:estimatebound} and~\ref{lem:rmbound}, respectively.


\subsection{Double-Optimistic Learning Analysis} \label{sec:doanalysis}
Lemma~\ref{lem:estimatebound} represents the estimation error of the cumulative reward-to-cost ratio.
Due to the page limit, we illustrate the key steps in the analysis of double-optimistic learning and the completed version can be found in Appendix C.

Recall that to optimize the reward-to-cost ratio given unknown reward and cost functions, we employ optimistic estimators for both rewards and costs.
This indicates that the following inequalities hold with a high probability according to UCB/LCU learning~\cite{TorCsa_20}
\begin{equation*}
    0 \leq \hat{r}_{S,a} - r_{S,a} \leq 0,~0 \leq  \frac{1}{\hat{c}_{S,a}} - \frac{1}{c_{S,a}} \leq 0. 
\end{equation*}
These guarantee optimistic estimation for cumulative terms $\sum_{t=1}^T r_{S_t,a_t}$ and $\frac{1}{\sum_{t=1}^T c_{S_t,a_t}}$, resulting in an optimistic estimation for cumulative reward-to-cost ratio $\frac{\sum_{t=1}^T\mathbb{E}[r_{S_t,a_t}]}{\sum_{t=1}^T\mathbb{E}[c_{S_t,a_t}]}$.

To proceed, we define \textit{partial double-optimistic errors}
\begin{align*}
    \epsilon(T) \triangleq & \left\vert\frac{\sum_{t=1}^T\mathbb{E}[\hat{r}_{S_t,a_t}]}{\sum_{t=1}^T\mathbb{E}[\check{c}_{S_t,a_t}]}-\frac{\sum_{t=1}^T\mathbb{E}[r_{S_t,a_t}]}{\sum_{t=1}^T\mathbb{E}[\check{c}_{S_t,a_t}]}\right\vert, \\
    \kappa(T) \triangleq &\left\vert\frac{\sum_{t=1}^T\mathbb{E}[r_{S_t,a_t}]}{\sum_{t=1}^T\mathbb{E}[\check{c}_{S_t,a_t}]}-\frac{\sum_{t=1}^T\mathbb{E}[r_{S_t,a_t}]}{\sum_{t=1}^T\mathbb{E}[c_{S_t,a_t}]}\right\vert.
\end{align*}

Such errors directly form an upper bound of the double-optimistic learning error in~\eqref{eq:estimate-real bound} via the triangle inequality:
\begin{equation*}
    \left\vert\frac{\sum_{t=1}^T\mathbb{E}[\hat{r}_{S_t,a_t}]}{\sum_{t=1}^T\mathbb{E}[\check{c}_{S_t,a_t}]}-\frac{\sum_{t=1}^T\mathbb{E}[r_{S_t,a_t}]}{\sum_{t=1}^T\mathbb{E}[c_{S_t,a_t}]}\right\vert \leq \epsilon(T) + \kappa(T),
\end{equation*}
We then bound these partial double-optimistic errors as follows:
\begin{align*}
    & \epsilon(T) + \kappa(T) \\
=&\frac{1}{\sum_{t=1}^T\mathbb{E}[\check{c}_{S_t,a_t}]}\left\vert\sum_{t=1}^T\mathbb{E}[\hat{r}_{S_t,a_t}]-\sum_{t=1}^T\mathbb{E}[r_{S_t,a_t}]\right\vert\\&+ \frac{\sum_{t=1}^T\mathbb{E}[r_{S_t,a_t}]}{\sum_{t=1}^T\mathbb{E}[\check{c}_{S_t,a_t}]\sum_{t=1}^T\mathbb{E}[c_{S_t,a_t}]}\left\vert\sum_{t=1}^T\mathbb{E}[c_{S_t,a_t}]-\sum_{t=1}^T\mathbb{E}[\check{c}_{S_t,a_t}]\right\vert\\
\leq& \frac{1}{Tc_{\min}}\left\vert\sum_{t=1}^T\mathbb{E}[\hat{r}_{S_t,a_t} -r_{S_t,a_t}]\right\vert+ \frac{r_{\max}}{Tc_{\min}^2}\left\vert\sum_{t=1}^T\mathbb{E}[c_{S_t,a_t} - \check{c}_{S_t,a_t}]\right\vert,
\end{align*}
where the last inequality holds due to the boundedness of reward and cost.
As shown in the above inequality, we break down the estimated error of cumulative ratio into the estimated errors of rewards and costs.
Consequently, we only need to bound the following terms 
\begin{align*}
  \left\vert\sum_{t=1}^T\mathbb{E}[\hat{r}_{S_t,a_t} -r_{S_t,a_t}]\right\vert,~\left\vert\sum_{t=1}^T\mathbb{E}[c_{S_t,a_t} - \check{c}_{S_t,a_t}]\right\vert,
\end{align*}
which have been widely studied in UCB/LCB learning~\cite{TorCsa_20,LiuLiBShi_21, LiuLiBShiLei_21} and both are in the order of $O(\sqrt{T})$.
Eventually, we have 
\begin{equation*}
    \epsilon(T) + \kappa(T) \leq O(T^{-\frac{1}{2}}),
\end{equation*}
which proves Lemma \ref{lem:estimatebound}.


\subsection{Robbins-Monro Iteration Analysis} \label{sec:rmanalysis}
Lemma~\ref{lem:rmbound} is the key to establishing the convergence gap and regret for DOL-RM and is also the most challenging part.
We first establish its upper bound with the terms related to $\mathbb{E}[(\theta_t-\theta^*)^2]$ in the following Lemma~\ref{lem:rm-decomp},
we then carefully control the cumulative errors in the Robbins-Monro iteration.
\begin{lemma}\label{lem:rm-decomp}

Suppose Assumption~\ref{assumption:reward} and~\ref{assumption:cost} hold, under DOL-RM, we bound the Robbins-Monro iteration convergence gap in~\eqref{eq:optimal-estimate gap} as follows:
\begin{align}
    &\left\vert\theta^*-\frac{\sum_{t=1}^T\mathbb{E}[\hat{r}_{S_t,a_t}]}{\sum_{t=1}^T\mathbb{E}[\check{c}_{S_t,a_t}]}\right\vert \nonumber \\
    &\leq \max \left\{ \left\vert\frac{\sum_{t=1}^T\mathbb{E}[\hat{r}_{S_t,a_t}]}{\sum_{t=1}^T\mathbb{E}[\check{c}_{S_t,a_t}]}-\frac{\sum_{t=1}^T\mathbb{E}[r_{S_t,a_t}]}{\sum_{t=1}^T\mathbb{E}[c_{S_t,a_t}]}\right\vert, \right. \label{eq:rm-decomp1} \\
    &\phantom{\leq \max \{} \left. \frac{\sqrt{2C_1}}{Tc_{\min}}\sum\limits_{t=1}^{T}\sqrt{\mathbb{E}\left[(\theta_t-\theta^*)^2\right]} \right\}. \label{eq:rm-decomp2}
\end{align}
\end{lemma}
We provide the proof sketch by considering two cases:
1) when $\theta^* \leq \frac{\sum_{t=1}^T\mathbb{E}[\hat{r}_{S_t,a_t}]}{\sum_{t=1}^T\mathbb{E}[\check{c}_{S_t,a_t}]}$, the term in~\eqref{eq:rm-decomp1} is upper bound of \eqref{eq:optimal-estimate gap} according to the definition of $\theta^*$.
2) when $\frac{\sum_{t=1}^T\mathbb{E}[\hat{r}_{S_t,a_t}]}{\sum_{t=1}^T\mathbb{E}[\check{c}_{S_t,a_t}]} \leq \theta^*$, we have the greedy decision in our algorithm such that
\begin{align*}
    \mathbb{E}[\hat{r}_{S_t,a_t}-\theta_t\check{c}_{S_t,a_t}\vert\theta_t] 
    \geq& \mathbb{E}[R_{S_t,a^*}\vert\theta_t]-\theta_t\mathbb{E}[C_{S_t,a^*}\vert\theta_t]  \\
    \geq& r^* - \theta_t c^*,
\end{align*} 
where $a^*$ is any decision within $\mathcal A(S_t)$, the first inequality holds due to the greedy decision and $\hat{r}_{S_t,a_t} \geq {r}_{S_t,a_t}, \check{c}_{S_t,a_t} \leq c_{S_t,a_t}$ holds with a high probability;
the second inequality holds because $(r^*, c^*)$ lies within the closure of available decision set.
By adding $(\theta_t - \theta^*)\check{c}_{S_t,a_t}$ on both sides of the above inequality, we further have
\begin{equation*}
    \mathbb{E} [\hat{r}_{S_t,a_t}-\theta^*\check{c}_{S_t,a_t}]\geq\mathbb{E}\left[(\theta_t-\theta^*)(\check{c}_{S_t,a_t}-c^*)\right].
\end{equation*}
Finally, one could use Cauchy-Schwarz inequality to establish~\eqref{eq:rm-decomp2} in Lemma~\ref{lem:rm-decomp}.
More details can be found in the Appendix A.

Since we already have the $O(T^{-\frac{1}{2}})$ bound of \eqref{eq:rm-decomp1} in Section~\ref{sec:doanalysis}, the key is to establish the cumulative error
\begin{equation*}
    \sum_{t=1}^{T}\sqrt{\mathbb{E}[(\theta_t-\theta^*)^2]}.
\end{equation*}
We leverage the Lyapunov drift analysis~\cite{Nee_10, SriYin_14} to study this key term.
Note that~\cite{Nee_21} also used this technique to study the cumulative error term, where they assumed the perfect information of rewards and costs such that it can establish anytime error of $\mathbb E[(\theta_t-\theta^*)^2]$ via one-step Lyapunov drift.
However, we do not assume any of such knowledge and quantify the cumulative error term by aggregating the bias over all tasks.

Define the Lyapunov function
\begin{equation*}
    V_t \triangleq \frac{1}{2}(\theta_t - \theta^*)^2,
\end{equation*}
we analyze the corresponding Lyapunov drift
\begin{equation*}
    \Delta(t) \triangleq V_{t+1} - V_t
\end{equation*}
and have the following lemma. 
\begin{lemma}\label{lem:drift}
    Under DOL-RM, we have the expected Lyapunov drift to be bounded as follows:
    \begin{align*}
        \mathbb{E}[\Delta(t) ] \leq&-2c_{\min}\eta \mathbb{E} [V_t]+\eta ^2 b \\ 
        &+\theta_{gap}\eta  \mathbb{E}\left[\hat{r}_{S_t,a_t} - r_{S_t,a_t}-\theta^*(\check{c}_{S_t,a_t} - c_{S_t,a_t})\right],
    \end{align*}
    where $\theta_{gap} \triangleq (\theta_{max} - \theta_{min})$ and $b \geq \frac{1}{2}\mathbb{E}\left[(\hat{r}_{S_t,a_t}-\theta_t\check{c}_{S_t,a_t})^2\right], \forall t$.
\end{lemma}
We in the following provide the proof sketch for this key lemma. 

According to the reward-to-cost ratio learning in~\eqref{eq:vqupdate}, we have
\begin{align*}
    (\theta_{t+1}-\theta^*)^2 
    &=\left(\left[\theta_t + \eta (\hat{r}_{S_t,a_t} - \theta_t \check{c}_{S_t,a_t})\right]_{\theta_{\min}}^{\theta_{\max}} - \theta^*\right)^2\\
    &= \left[ \left(\theta_t + \eta (\hat{r}_{S_t,a_t} - \theta_t \check{c}_{S_t,a_t}) - \theta^*\right)_{\theta_{\min}}^{\theta_{\max}}\right]^2\\
    &\leq (\theta_t + \eta (\hat{r}_{S_t,a_t} - \theta_t \check{c}_{S_t,a_t}) - \theta^*)^2.
\end{align*}
By the definition of the Lyapunov drift $\Delta_t$, we have
\begin{align*}
    \mathbb{E}[\Delta(t)] =& \frac{1}{2}\mathbb{E}[(\theta_{t+1}-\theta^*)^2] - \frac{1}{2}\mathbb{E}[(\theta_{t}-\theta^*)^2]\\
    \leq& \eta^2 b +\eta \mathbb{E}\left[(\theta_t - \theta^*)(\hat{r}_{S_t,a_t} - \theta_t \check{c}_{S_t,a_t})\right].
\end{align*}
Similar to the proof of Lemma~\ref{lem:rm-decomp}, we bound the last term in the above inequality under two cases:

1) For $\theta^* \geq \theta_t$, we can directly get
\begin{align*}
    \mathbb{E}\left[(\theta_t - \theta^*)(\hat{r}_{S_t,a_t} - \theta_t \check{c}_{S_t,a_t})\right] \leq -2c_{\min}\eta \mathbb{E} [V_t],
\end{align*}
since $\mathbb{E}[\hat{r}_{S_t,a_t}-\theta_t\check{c}_{S_t,a_t}\vert\theta_t] \geq r^* - \theta_t c^* = c^*(\theta^* - \theta_t)$.
Then the inequality in Lemma \ref{lem:drift} holds by the fact $ \mathbb{E} [ \hat{r}_{S_t, a_t} - r_{S_t, a_t}] \geq 0, \mathbb{E} [c_{S_t, a_t} -  \check{c}_{S_t, a_t} ] \geq 0 $ which is guaranteed by the property of UCB/LCB.

2) For $\theta^* \leq \theta_t$, given the boundedness of cost, we have
\begin{align*}
    &\mathbb{E}\left[(\theta_t - \theta^*)(\hat{r}_{S_t,a_t} - \theta_t \check{c}_{S_t,a_t})\right] \\
    \leq& \mathbb{E}\left[(\theta_t - \theta^*)(\hat{r}_{S_t,a_t} - \theta^* \check{c}_{S_t,a_t})\right] - c_{min}\mathbb{E}\left[(\theta_t - \theta^*)^2\right].
\end{align*}

We then have the following bound by decomposing the first term and incorporating the fact that $\theta_t - \theta^* \leq \theta_{gap}$, i.e., 
\begin{align*}
    &\mathbb{E}\left[(\theta_t - \theta^*)(\hat{r}_{S_t,a_t} - \theta^* \check{c}_{S_t,a_t})\right] \\
    \leq &\mathbb{E}\left[(\theta_t - \theta^*)(r_{S_t,a_t} - \theta^* c_{S_t,a_t})\right] \\
    &+\theta_{gap} \mathbb{E} \left[(\hat{r}_{S_t,a_t} - r_{S_t,a_t}-\theta^*(\check{c}_{S_t,a_t} 
    - c_{S_t,a_t}))\right].
\end{align*}

Next, drawing upon the definition of the optimal ratio $\theta^*$, which indicates that $\mathbb{E}\left[(r_{S_t,a_t} - \theta^* c_{S_t,a_t})\right] \leq 0,$ we complete the proof of Lemma \ref{lem:drift}.
More details are provided in Appendix B.

{\bf Proving Lemma~\ref{lem:rmbound}:} 
We rearrange the inequality in Lemma~\ref{lem:drift} and take summation from $1$ to $T$ over it, yielding the following:
\begin{align*}
    \sum_{t=1}^T \mathbb{E} [V_t] \leq& -\sum_{t=1}^T \frac{\mathbb{E}[\Delta(t)]}{2c_{\min}\eta } + \frac{T\eta  b}{2c_{min}} \\
    &+ \frac{\theta_{gap}}{2c_{min}} \sum_{t=1}^T  \mathbb{E}\left[\hat{r}_{S_t,a_t} - r_{S_t,a_t}-\theta^*(\check{c}_{S_t,a_t} - c_{S_t,a_t})\right]\\
    \leq& \frac{V_1}{2c_{\min}\eta } + \frac{T\eta  b}{2c_{min}} + O(\sqrt{T}) = O(\sqrt{T}),
\end{align*}
where the second inequality holds because of $\sum_{t=1}^T \Delta(t) = V_{T}-V_1$ and the fact that the estimation errors are at the order of $O(\sqrt{T})$ according to UCB/LCB learning (similarly stated in Section~\ref{sec:doanalysis});
the last equality holds by plugging the learning rate $\eta=O(1/\sqrt{T})$.

According to the Cauchy-Schwarz inequality, we have,
\begin{equation*}
    \sum_{t=1}^T \sqrt{\mathbb{E}[V_t]} \leq \sqrt{T \sum_{t=1}^T \mathbb{E}[V_t]},
\end{equation*}
which implies that the key cumulative error is at the order of $O(T^{\frac{3}{4}})$ in the following lemma.
\begin{lemma}\label{lem:sqrtsumbound}
    Under DOL-RM, we have
    \begin{equation*}
        \sum\limits_{t=1}^{T}\sqrt{\mathbb{E}\left[(\theta_t-\theta^*)^2\right]} = O(T^{\frac{3}{4}}).
    \end{equation*}
\end{lemma}
Finally, by combining Lemma~\ref{lem:estimatebound} and~\ref{lem:sqrtsumbound} into Lemma~\ref{lem:rm-decomp}, we establish the Robbins-Monro iteration convergence gap as follows:
\begin{equation*}
    \left\vert\theta^*-\frac{\sum_{t=1}^T\mathbb{E}[\hat{r}_{S_t,a_t}]}{\sum_{t=1}^T\mathbb{E}[\check{c}_{S_t,a_t}]}\right\vert 
    \leq \max \left\{ O(T^{-1/2}), O(T^{-1/4})\right\},
\end{equation*}
which proves Lemma~\ref{lem:rmbound}.

\begin{figure*}[!t]
    \centering
    \begin{subfigure}[Two-Task-Type $p=0.8$.]{
        \includegraphics[width=0.31\textwidth]{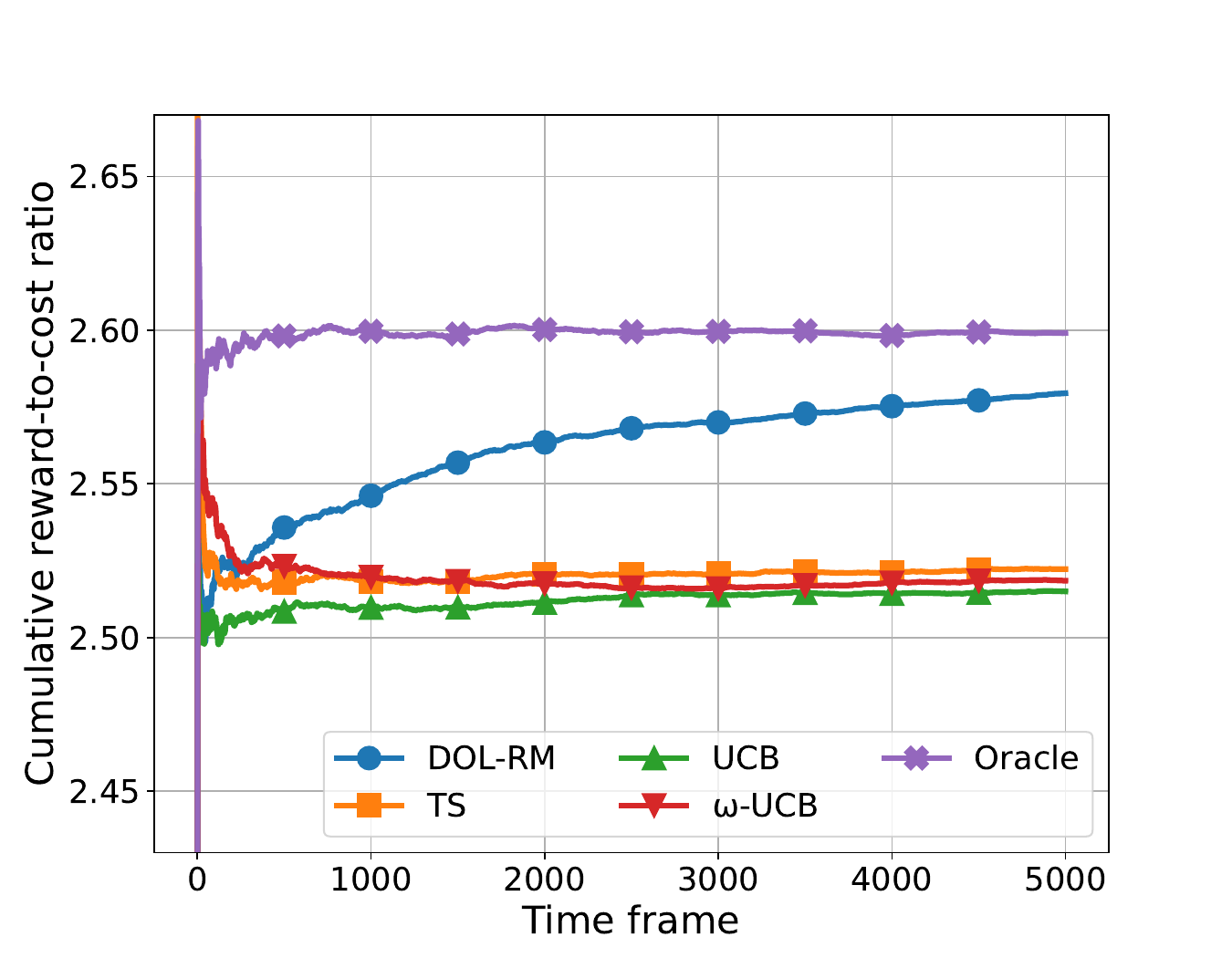}\label{fig:synthetic exp1}}
    \end{subfigure}
    \begin{subfigure}[Two-Task-Type $p=0.6$.]{
        \includegraphics[width=0.31\textwidth]{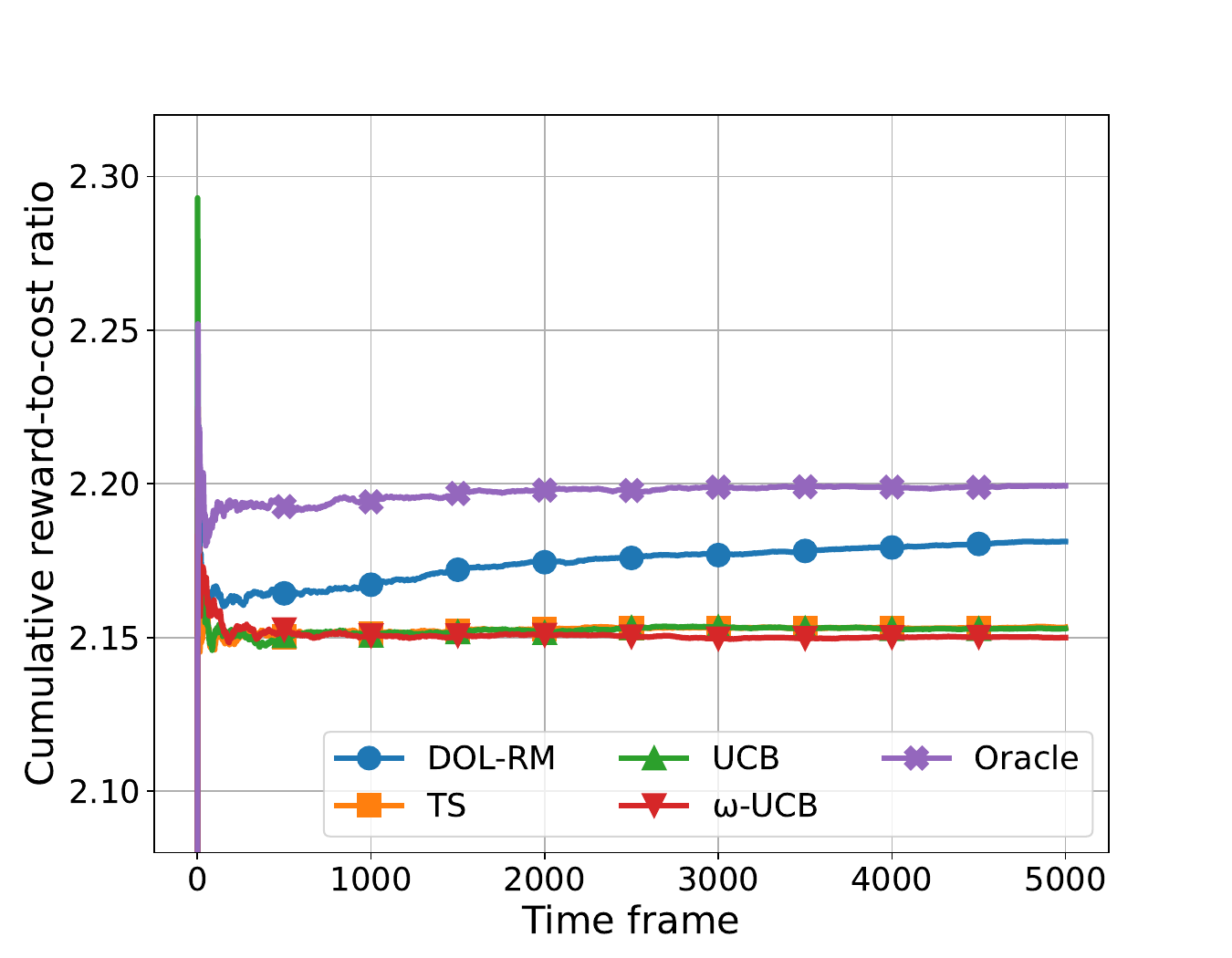}\label{fig:synthetic exp2}}
    \end{subfigure}
    \begin{subfigure}[Seven-Task-Type.]{
        \includegraphics[width=0.31\textwidth]{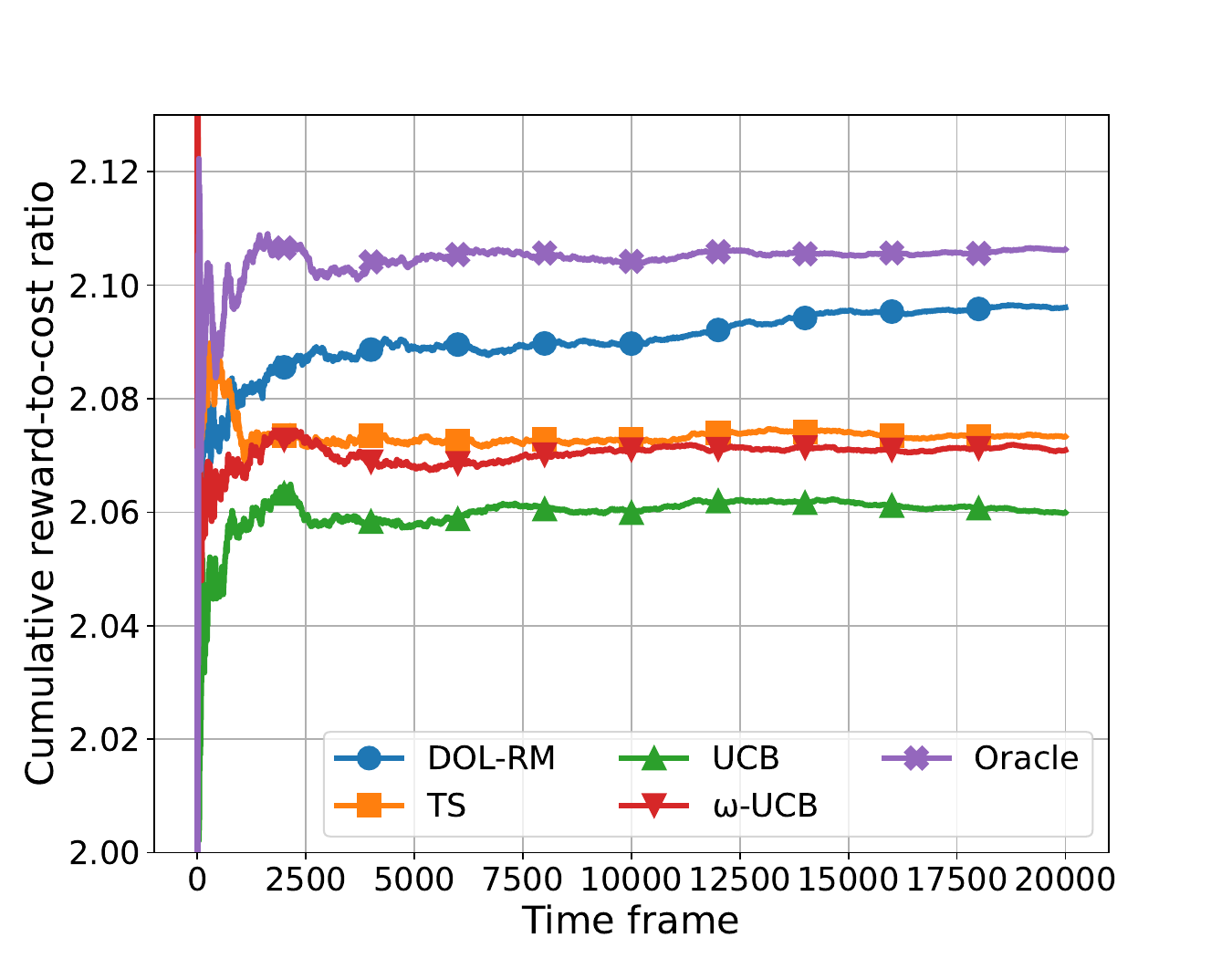}\label{fig:synthetic data}}
    \end{subfigure}
    \caption{Synthetic Experiment: Performance Comparison in a Two-Task-Type Case and Seven-Task-Type Case.}
    \label{fig:synthetic exp}
\end{figure*}


\begin{figure}[!t]
    \centerline{\includegraphics[width=0.35\textwidth]{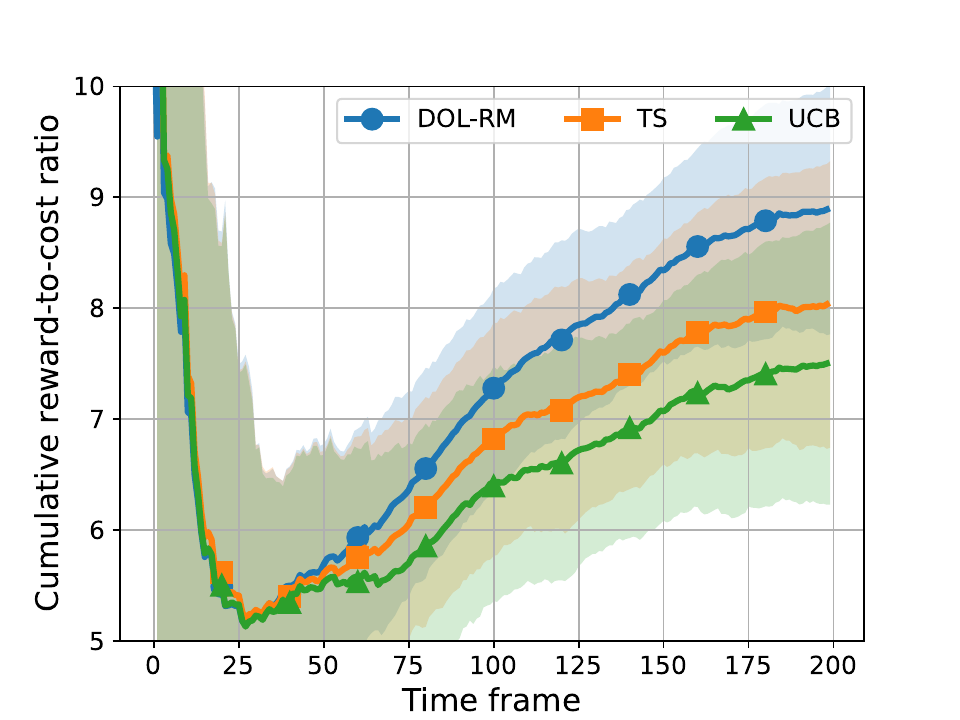}}
    \caption{Real-World Experiment.}
    \label{fig:real data}
\end{figure}

\section{Numerical Results}
We evaluate DOL-RM via both synthetic simulation and real-world experiment in terms of the cumulative reward-to-cost ratio $\frac{\sum_1^t R_{S_t, a_t}}{\sum_1^t C_{S_t, a_t}}$.
For both cases, we let learning rate $\eta_t = 1/{c_{min}(t+1)}$ in our DOL-RM algorithm. 

{\noindent \bf Baselines:}
We consider the following representative baselines to justify DOL-RM's performance:
\begin{itemize}
    \item General Thompson Sampling (TS)~\cite{DanBenAbb_18}.
    \item Classic Upper Confidence Bound (UCB) algorithm~\cite{TorCsa_20}.
    \item $\omega$-UCB~\cite{MarVadEdo_23}, a variance-aware cost-efficient algorithm.
    \item Oracle algorithm~\cite{Nee_21}, which needs prior knowledge of reward and cost and an optimal solution\footnote{We adopt this algorithm as an ideal benchmark as it possesses full knowledge before decision-making, allowing us to verify whether DOL-RM converges to the optimal ratio. However, it cannot be applied in the real world due to the lack of prior knowledge.}.
\end{itemize}

\subsection{Synthetic Simulation}
\textbf{\textbullet Two-Task-Type}
We consider a synthetic setting similar to the toy example presented in \figurename~\ref{fig:Example}.
We consider two incoming typs of tasks with arrival rates of $\{p, 1-p\}$ and reward-cost vectors of $\{[(3,1)],[(3,2),(1,1)]\}$.
The observations $R_{S, a}$ and $C_{S, a}$ are corrupted with additive Gaussian noise sampled from $\mathcal{N}(0,1)$.

We let $p$ and $1-p$ be the arrival probability of task $S_x$ and $S_y$, respectively.
To verify the robustness of DOL-RM against varied task arrival distributions, we test it under two task arrival patterns, quantified by $p=0.8$, and $p=0.6$ in~\figurename~\ref{fig:synthetic exp1} and ~\ref{fig:synthetic exp2}.

\textbf{\textbullet Seven-Task-Type}
We consider a more complex task scheduling problem with 7 types of incoming tasks with arrival rates of $\{0.3,0.1,0.2,0.1,0.05,0.1,0.15\}$ and reward-cost vectors of $\{[(3,1)],$ $[(3,2),(1,1)],$ $[(2,1)],$ $[(2.5,1.5)],$ $[(2,1),(1,1)],$ $[(3,2),$ $(1.5,1.5)],$ $[(2.5,1)]\}$. The observations $R_{S, a}$ and $C_{S, a}$ are corrupted with additive Gaussian noise sampled from $\mathcal{N}(0, 1)$.

\figurename~\ref{fig:synthetic exp} demonstrates that DOL-RM outperforms the state-of-the-art learning-based algorithms in terms of the cumulative reward-to-cost ratio.
DOL-RM's superior performance is consistent across two different cases, showcasing its high adaptability.
Compared to the oracle algorithm~\cite{Nee_21}, which has full prior knowledge and the optimal ratio in hindsight, DOL-RM exhibits a fast convergence rate, indicating its efficiency in identifying the optimal policy under certainty.


\subsection{Real-World Experiment}

We apply DOL-RM to schedule machine learning (ML) training tasks in a shared server. 
Specifically, we consider five types of ML classification tasks in the system, where a satellite image classification task $S_{s}$ with four categories
, a weather classification task $S_{w}$ and a rice classification task $S_{r}$ with five categories 
, a natural scene classification task $S_{n}$ with six categories 
and a cat \& dog classification task $S_{c}$.
Let $\mathcal{S} = \{ S_{s}, S_{w}, S_{r}, S_{n}, S_{c} \}$.
For each classification task, the server decides a training epoch from a range of options.
The available options on epoch number are within $\mathcal{A}(S_s) = \mathcal{A}(S_r) =\mathcal{A}(S_c) = \{1, 2, 3,$ $ 4, 5\}$ and $\mathcal{A}(S_w) = \mathcal{A}(S_n) = \{1, 5, 10, 15, 20\}$\footnote{For satellite image classification $S_{s}$, each epoch contains $141$ steps. For weather classification $S_{w}$, each epoch contains $38$ steps. For rice classification $S_{r}$, each epoch contains $100$ steps. For natural scene classification $S_{n}$, each epoch contains $150$ steps. For cat \& dog classification $S_{c}$, each epoch has $30$ steps.}, respectively.
After training the model with an epoch number $a_t \in \mathcal{A}(S_{t})$ for a classification task $S_t \in \mathcal{S}$, we can obtain the AUC on the test set as reward $R_{S_t,a_t}$ and observe the training time  as cost $C_{S_t,a_t}$\footnote{We conduct min-max normalization on reward and cost to facilitate practical training.}.

The dataset of satellite image classification has more than $1000$ pictures in each category of satellite remote-sensing images ~\cite{RED_21}.
All the pictures are resized into $255 \times 255$ and $20\%$ of them are separated as the test set.
The dataset of weather classification contains boasting over $300$ pictures of weather images in each category ~\cite{GUP_19}.
All the pictures are resized into $150 \times 150$ and we construct the test set with $20\%$ of the pictures.
We use the first $1000$ pictures of rice images in each class in ~\cite{KOK_22}.
The pictures are all resized into $224 \times 224$ and $20\%$ of them are reserved as the test set.
We take the first $500$ pictures in each category in the training set of ~\cite{BAN_19} as training data and utilize the test set of ~\cite{BAN_19} as test data.
All the pictures are resized into $150 \times 150$.
The database of cat \& dog classification contains $11875$ pictures for each class of cat or dog~\cite{SAC_19}, where the pictures are also scaled into size $150\times150$ and $10\%$ of these pictures are reserved as the test set. 
Our experiment was conducted on an \textit{RTX 3080 Ti} GPU, running a 64-bit Ubuntu 18.04 system.
The detailed structures and parameters of neural network models can be found in Appendix E.
Let the arrival probability of $\{S_{s}, S_{w}, S_{r}, S_{n}, S_{c}\}$ be $\{0.1, 0.4, 0.2,$ $ 0.2, 0.1\}$.
As $\omega$-UCB~\cite{MarVadEdo_23} and oracle algorithm~\cite{Nee_21} requires non-casual information, which is infeasible in the experiment, we compare DOL-RM only with the (classical) UCB and TS.
\figurename~\ref{fig:real data} plot the cumulative reward-to-cost ratio for these algorithms where the light-shaded areas indicate the corresponding standard deviation.
These results show DOL-RM outperforms the baselines significantly and demonstrate that DOL-RM can converge to a better policy in the real-world system even without any prior information. 



\section{Conclusion}

In this paper, we initiated the study on a novel formulation of the online task scheduling problem, where the task arrival distribution, rewards, and costs are all unknown.
Guided by double-optimistic learning and Robbins-Monro method, we proposed an effective and efficient algorithm DOL-RM which integrates optimistic estimations and stochastic approximation of balancing point.
We theoretically demonstrated its superior performance with a simultaneous achievement of sub-linear regret bound and fast learning.
Via justification in synthetic and real-world scenarios, we not only showed the outperformance of DOL-RM over state-of-the-art baselines but also envisioned enormous potential applications of our modeling and algorithm design.



\begin{acks}
The work was partly supported by the Shanghai Sailing Program 22YF1428500, the National Nature Science Foundation of China under grant 62302305. Corresponding author: Xin Liu.
\end{acks}



\bibliographystyle{ACM-Reference-Format} 
\balance
\bibliography{sample}


\clearpage
\onecolumn

\appendix

\section{Proof of Lemma~\ref{lem:rm-decomp}}
\label{app:decomposition proof}

Under the case that $\theta^* \leq \frac{\sum_{t=1}^T\mathbb{E}[\hat{r}_{S_t,a_t}]}{\sum_{t=1}^T\mathbb{E}[\check{c}_{S_t,a_t}]}$, the following inequality holds obviously,
\begin{equation*}
    \left\vert\theta^*-\frac{\sum_{t=1}^T\mathbb{E}[\hat{r}_{S_t,a_t}]}{\sum_{t=1}^T\mathbb{E}[\check{c}_{S_t,a_t}]}\right\vert \nonumber 
    \leq \left\vert\frac{\sum_{t=1}^T\mathbb{E}[\hat{r}_{S_t,a_t}]}{\sum_{t=1}^T\mathbb{E}[\check{c}_{S_t,a_t}]}-\frac{\sum_{t=1}^T\mathbb{E}[r_{S_t,a_t}]}{\sum_{t=1}^T\mathbb{E}[c_{S_t,a_t}]}\right\vert, 
\end{equation*}

We then focus on the case that $\theta^* \geq \frac{\sum_{t=1}^T\mathbb{E}[\hat{r}_{S_t,a_t}]}{\sum_{t=1}^T\mathbb{E}[\check{c}_{S_t,a_t}]}$, 
and we start with the following lemma.
\begin{lemma}[Lemma $3$ in~\cite{Nee_21}]
    Under DOL-RM, the following inequalities hold for all $t\in [T]$:
    \begin{equation*}
        \mathbb{E}[\hat{r}_{S_t,a_t}-\theta^*\check{c}_{S_t,a_t}\vert\theta_t]\geq (\theta_t-\theta^*)\mathbb{E}[\check{c}_{S_t,a_t}-c^*\vert\theta_t].
    \end{equation*}
\end{lemma}
\begin{proof}
According to the updating rule of $a_t$ in~\eqref{eq:decision}, we have
\begin{equation*}
    \hat{r}_{S_t,a_t} - \theta_t \check{c}_{S_t,a_t} \geq \hat{r}_{S_t,a} - \theta_t \check{c}_{S_t,a},~\forall a \in \mathcal{A}(S_t).
\end{equation*}

Take conditional expectations at both sides and we get,
\begin{equation}
    \mathbb{E}[\hat{r}_{S_t,a_t}-\theta_t\check{c}_{S_t,a_t}\vert\theta_t]\geq\mathbb{E}[\hat{r}_{S_t,a}\vert\theta_t]-\theta_t\mathbb{E}[\check{c}_{S_t,a}\vert\theta_t]. \label{eq:action ineq}
\end{equation}

Since $\hat{r}_{S_t,a}$ and $\check{c}_{S_t,a}$ serve as optimistic UCB/LCB estimators, the following inequalities hold with probability $1$,
\begin{equation}
    \mathbb{E}[\check{c}_{S_t,a}\vert\theta_t]\leq \mathbb{E}[C_{S_t,a}\vert\theta_t],~\mathbb{E}[\hat{r}_{S_t,a}\vert\theta_t] \geq \mathbb{E}[R_{S_t,a}\vert\theta_t]. \label{eq:ucb/lcb ineq}
\end{equation}

Substituting~\eqref{eq:ucb/lcb ineq} into~\eqref{eq:action ineq} yields
\begin{equation*}
    \mathbb{E}[\hat{r}_{S_t,a_t}-\theta_t\check{c}_{S_t,a_t}\vert\theta_t]\geq\mathbb{E}[R_{S_t,a}\vert\theta_t]-\theta_t\mathbb{E}[C_{S_t,a}\vert\theta_t].
\end{equation*}

Since $\theta_t$ only depends on the history in the system before $t$, it is independent of $S_t$.
Fix $(r, c) \in \mathcal{D}$, where $\mathcal{D}$ represents the set of all reward-cost pairs corresponding to all available decisions.
There is a conditional distribution for choosing $(R_{S_1,a_1}, C_{S_1,a_1})\in\mathcal{D}(S_1)$, given the observed $S_1$, such that $(\mathbb{E}[r_{S_1,a_1}], \mathbb{E}[c_{S_1,a_1}]) = (r,c)$.
Since $S_t$ is independent of $\theta_t$ and has the same distribution as $S_1$, we can use the same conditional distribution to get a result of a bandit $(r_{S_t,a^*},c_{S_t,a^*})\in \mathcal{D}(S_t)$ that is independent of $\theta_t$ such that
\begin{equation*}
    (\mathbb{E}[R_{S_t,a^*}],\mathbb{E}[C_{S_t,a^*}]) = (r,c).
\end{equation*}
Besides, since $(\mathbb{E}[R_{S_t,a^*}],\mathbb{E}[C_{S_t,a^*}])$ is independent of $\theta_t$, we have (with probability 1)
\begin{equation*}
    \mathbb{E}[R_{S_t,a^*}\vert\theta_t] = r, ~\mathbb{E}[C_{S_t,a^*}\vert\theta_t] = c,
\end{equation*}
which indicates that
\begin{equation}
    \mathbb{E}[\hat{r}_{S_t,a_t}-\theta_t\check{c}_{S_t,a_t}\vert\theta_t]\geq r-\theta_t c. \label{e6}
\end{equation}

Fix $(r^*,c^*)\in\bar{\mathcal{D}}$ that satisfies $\theta^*=\frac{r^*}{c^*}$.
Since $\bar{\mathcal{D}}$ is the closure of the set $\mathcal{D}$, there is a sequence of points $\{(r_i,c_i)\}_{i=1}^{\infty}$ in $\mathcal{D}$ that converge to the value $(r^*,c^*)$.
Since~\eqref{e6} holds for an arbitrary $(r,c)\in \mathcal{D}$, with probability 1, it holds simultaneously for all of $(r_i,c_i) \in \mathcal{D}$ for $i\in\{1,2,\dots\}$.
Therefore, we have
\begin{equation*}
    \mathbb{E}[\hat{r}_{S_t,a_t}-\theta_t\check{c}_{S_t,a_t}\vert\theta_t]\geq r_i-\theta_t c_i.
\end{equation*}

Taking a limit as $i\rightarrow\infty$ yields (with probability 1):
\begin{equation}
    \mathbb{E}[\hat{r}_{S_t,a_t}-\theta_t\check{c}_{S_t,a_t}\vert\theta_t] \geq r^*-\theta_t c^* = c^*(\theta^*-\theta_t). \label{eq:another form of support lemma}
\end{equation}

Adding $(\theta_t-\theta^*)\mathbb{E}[\check{c}_{S_t,a_t}\vert\theta_t]$ to both sides implies:
\begin{equation*}
    \mathbb{E}[\hat{r}_{S_t,a_t}-\theta^*\check{c}_{S_t,a_t}\vert\theta_t]\geq (\theta^* - \theta_t)c^* + (\theta_t-\theta^*)\mathbb{E}[\check{c}_{S_t,a_t}\vert\theta_t].
\end{equation*}

Taking expectations on both sides, we have 
\begin{equation*}
    \mathbb{E}[\hat{r}_{S_t,a_t}-\theta^*\check{c}_{S_t,a_t}\vert\theta_t]\geq\mathbb{E}\left[(\theta_t-\theta^*)(\check{c}_{S_t,a_t}-c^*)\vert\theta_t\right].
\end{equation*}

\end{proof}

Takning expectation on both sides of the above inequality, we have
\begin{align*}
    \mathbb{E}[\hat{r}_{S_t,a_t}-\theta^*\check{c}_{S_t,a_t}]\geq\mathbb{E}\left[(\theta_t-\theta^*)(\check{c}_{S_t,a_t}-c^*)\right].
\end{align*}
Take summation and rearrange the above inequality, we have
\begin{align*}
    \frac{\sum_{t=1}^T\mathbb{E}[\hat{r}_{S_t,a_t}]}{\sum_{t=1}^T\mathbb{E}[\check{c}_{S_t,a_t}]}
    &\geq\theta^* + \frac{1}{{\sum_{t=1}^T\mathbb{E}[\check{c}_{S_t,a_t}]}}\sum_{t=1}^T\mathbb{E}\left[(\theta_t-\theta^*)(\check{c}_{S_t,a_t}-c^*)\right]\\
    &\geq \theta^* - \frac{1}{Tc_{\min}}\sum_{t=1}^T\mathbb{E}[\vert\theta_t-\theta^*\vert\vert\check{c}_{S_t,a_t}-c^*\vert],
\end{align*}
where the last inequality comes from the truncation of $\check{c}_{S_t,a_t}$.
According to the condition $\theta^*\geq\frac{\sum_{t=1}^T\mathbb{E}[\hat{r}_{S_t,a_t}]}{\sum_{t=1}^T\mathbb{E}[\check{c}_{S_t,a_t}]}$, we prove the following inequality which completes the proof of Lemma~\ref{lem:rm-decomp},
\begin{align}
    \left\vert\theta^*-\frac{\sum_{t=1}^T\mathbb{E}[\hat{r}_{S_t,a_t}]}{\sum_{t=1}^T\mathbb{E}[\check{c}_{S_t,a_t}]}\right\vert &= \theta^*-\frac{\sum_{t=1}^T\mathbb{E}[\hat{r}_{S_t,a_t}]}{\sum_{t=1}^T\mathbb{E}[\check{c}_{S_t,a_t}]} \notag\\
    &\leq \frac{1}{Tc_{\min}}\sum_{t=1}^T\mathbb{E}[\vert\theta_t-\theta^*\vert\vert\check{c}_{S_t,a_t}-c^*\vert]\\
    &\overset{(a)}{\leq} \frac{1}{Tc_{\min}}\sum_{t=1}^T\sqrt{\mathbb{E}\left[(\theta_t-\theta^*)^2(\check{c}_{S_t,a_t}-c^*)^2\right]}\\
    &\overset{(b)}{\leq}\frac{\sqrt{2C_1}}{Tc_{\min}}\sum\limits_{t=1}^{T}\sqrt{\mathbb{E}\left[(\theta_t-\theta^*)^2\right]},
    \label{e11}
\end{align}
where $(a)$ follows by the Cauchy-Schwarz inequality;
$(b)$ holds because $(\check{c}_{S_t,a_t}-c^*)^2\leq \check{c}_{S_t,a_t}^2 + (c^*)^2 \leq 2C_1^2$ with $(C_1 \geq \mathbb{E}[c_{S, a}^2],~\forall a \in \mathcal{A}(S))$.

\section{Proof of Lemma~\ref{lem:drift}} \label{app: drift}

According to virtual queue update of $\theta_t$ in~\eqref{eq:vqupdate}, we have
\begin{align*}
    (\theta_{t+1}-\theta^*)^2 
    &=\left(\left[\theta_t + \eta (\hat{r}_{S_t,a_t} - \theta_t \check{c}_{S_t,a_t})\right]_{\theta_{\min}}^{\theta_{\max}} - \theta^*\right)^2\\
    &=\left(\left[\theta_t + \eta (\hat{r}_{S_t,a_t} - \theta_t \check{c}_{S_t,a_t})\right]_{\theta_{\min}}^{\theta_{\max}} - [\theta^*]_{\theta_{\min}}^{\theta_{\max}}\right)^2\\
    &= \left[ \left(\theta_t + \eta (\hat{r}_{S_t,a_t} - \theta_t \check{c}_{S_t,a_t}) - \theta^*\right)_{\theta_{\min}}^{\theta_{\max}}\right]^2\\
    &\leq (\theta_t + \eta (\hat{r}_{S_t,a_t} - \theta_t \check{c}_{S_t,a_t}) - \theta^*)^2.
\end{align*}

Recall the definition $V_t = \frac{1}{2} (\theta_t - \theta^*)^2$, we have
\begin{align*}
    V_{t+1}
    &\leq \frac{1}{2}((\theta_t - \theta^*) + \eta (\hat{r}_{S_t,a_t} - \theta_t \check{c}_{S_t,a_t}))^2\notag\\
    &= V_t + \frac{\eta ^2}{2}(\hat{r}_{S_t,a_t} - \theta_t \check{c}_{S_t,a_t})^2+\eta (\theta_t - \theta^*)(\hat{r}_{S_t,a_t} - \theta_t \check{c}_{S_t,a_t}).
\end{align*}

Taking expectations on both sides, we have
\begin{equation*}
    \mathbb{E}\left[V_{t+1}\right] \leq \mathbb{E}\left[V_t\right] + \eta ^2b+ \eta \mathbb{E}\left[(\theta_t - \theta^*)(\hat{r}_{S_t,a_t} - \theta_t \check{c}_{S_t,a_t})\right],
\end{equation*}
where $b \geq \frac{1}{2}\mathbb{E}\left[(\hat{r}_{S_t,a_t}-\theta_t\check{c}_{S_t,a_t})^2\right], \forall t$.

To prove Lemma~\ref{lem:drift}, it suffices to show
\begin{equation*}
    \mathbb{E}\left[(\theta_t - \theta^*)(\hat{r}_{S_t,a_t} - \theta_t \check{c}_{S_t,a_t})\right] \leq -c_{\min}\mathbb{E}\left[(\theta_t-\theta^*)^2\right] + \theta_{gap}\mathbb{E} \left[\hat{r}_{S_t,a_t} - r_{S_t,a_t}-\theta^*(\check{c}_{S_t,a_t} - c_{S_t,a_t})\right].
\end{equation*}

We consider the proof under two cases separately:
\begin{itemize}
    \item
    When $\theta_t - \theta^* < 0$, take expectation of~\eqref{eq:another form of support lemma} we can straightly get
    \begin{align*}
        \mathbb{E}\left[(\theta_t - \theta^*)(\hat{r}_{S_t,a_t} - \theta_t \check{c}_{S_t,a_t})\right] 
        &\leq -c^*\mathbb{E}\left[(\theta_t-\theta^*)^2\right]\notag\\
        &\leq -c_{\min}\mathbb{E}\left[(\theta_t-\theta^*)^2\right] + \theta_{gap} \mathbb{E} \left[\hat{r}_{S_t,a_t} - r_{S_t,a_t}-\theta^*(\check{c}_{S_t,a_t} - c_{S_t,a_t})\right],
    \end{align*}
    where the last inequality holds since $c^* \geq c_{\min}$ and $\mathbb{E} \left[\hat{r}_{S_t,a_t} - r_{S_t,a_t}-\theta^*(\check{c}_{S_t,a_t} 
    - c_{S_t,a_t})\right] \geq 0$.
    \item
    When $\theta_t - \theta^* \geq 0$, given the boundedness of cost:
    \begin{align*}
        \mathbb{E}\left[(\theta_t - \theta^*)(\hat{r}_{S_t,a_t} - \theta_t \check{c}_{S_t,a_t})\right]
        =& \mathbb{E}\left[(\theta_t - \theta^*)(\hat{r}_{S_t,a_t} - \theta^* \check{c}_{S_t,a_t})\right] - \check{c}_{S_t,a_t}\mathbb{E}\left[(\theta_t - \theta^*)^2\right]\\
        \leq& \mathbb{E}\left[(\theta_t - \theta^*)(\hat{r}_{S_t,a_t} - \theta^* \check{c}_{S_t,a_t})\right] - c_{min}\mathbb{E}\left[(\theta_t - \theta^*)^2\right].
    \end{align*}
    For the first term, we have
    \begin{align*}
        \mathbb{E}\left[(\theta_t - \theta^*)(\hat{r}_{S_t,a_t} - \theta^* \check{c}_{S_t,a_t})\right]
        =&\mathbb{E}\left[(\theta_t - \theta^*)(r_{S_t,a_t} - \theta^* c_{S_t,a_t})\right] 
        + \mathbb{E} \left[(\theta_t - \theta^*)(\hat{r}_{S_t,a_t} - r_{S_t,a_t}-\theta^*(\check{c}_{S_t,a_t} 
        - c_{S_t,a_t}))\right]. \\
        &\leq \mathbb{E}\left[(\theta_t - \theta^*)(r_{S_t,a_t} - \theta^* c_{S_t,a_t})\right] 
        +\theta_{gap} \mathbb{E} \left[(\hat{r}_{S_t,a_t} - r_{S_t,a_t}-\theta^*(\check{c}_{S_t,a_t} 
        - c_{S_t,a_t}))\right].
    \end{align*}
    Recall the definition of optimal ratio $\theta^*$, we can get
    \begin{equation*}
        \mathbb{E}[r_{S_t,a_t}\vert\theta_t] - \theta^* \mathbb{E}[c_{S_t,a_t}\vert\theta_t] 
        \leq \sup\limits_{(r,c)\in \bar{\mathcal{D}}}\{r-\theta^*c\}=0,
    \end{equation*}
    which indicates that 
    \begin{equation*}
        \mathbb{E}\left[(\theta_t - \theta^*)(r_{S_t,a_t} - \theta^* c_{S_t,a_t})\right] \leq 0.
    \end{equation*}
    Combine these inequalities and we complete the proof as follows:
    \begin{align*}
        \mathbb{E}\left[(\theta_t - \theta^*)(\hat{r}_{S_t,a_t} - \theta_t \check{c}_{S_t,a_t})\right] 
        \leq -c_{\min}\mathbb{E}\left[(\theta_t-\theta^*)^2\right] + \theta_{gap}\mathbb{E} \left[\hat{r}_{S_t,a_t} - r_{S_t,a_t}-\theta^*(\check{c}_{S_t,a_t} - c_{S_t,a_t})\right].
    \end{align*}
\end{itemize}

\section{Proof of Lemma \ref{lem:estimatebound}} \label{app:double optimistic proof}

Applying the triangle inequality on~\eqref{eq:rm-decomp1}, we have
\begin{align*}
    &\left\vert\frac{\sum_{t=1}^T\mathbb{E}[\hat{r}_{S_t,a_t}]}{\sum_{t=1}^T\mathbb{E}[\check{c}_{S_t,a_t}]}-\frac{\sum_{t=1}^T\mathbb{E}[r_{S_t,a_t}]}{\sum_{t=1}^T\mathbb{E}[c_{S_t,a_t}]}\right\vert \\
    \leq &\left\vert\frac{\sum_{t=1}^T\mathbb{E}[\hat{r}_{S_t,a_t}]}{\sum_{t=1}^T\mathbb{E}[\check{c}_{S_t,a_t}]}-\frac{\sum_{t=1}^T\mathbb{E}[r_{S_t,a_t}]}{\sum_{t=1}^T\mathbb{E}[\check{c}_{S_t,a_t}]}\right\vert+\left\vert\frac{\sum_{t=1}^T\mathbb{E}[r_{S_t,a_t}]}{\sum_{t=1}^T\mathbb{E}[\check{c}_{S_t,a_t}]}-\frac{\sum_{t=1}^T\mathbb{E}[r_{S_t,a_t}]}{\sum_{t=1}^T\mathbb{E}[c_{S_t,a_t}]}\right\vert \\
    =&\frac{1}{\sum_{t=1}^T\mathbb{E}[\check{c}_{S_t,a_t}]}\left\vert\sum_{t=1}^T\mathbb{E}[\hat{r}_{S_t,a_t}]-\sum_{t=1}^T\mathbb{E}[r_{S_t,a_t}]\right\vert+ \frac{\sum_{t=1}^T\mathbb{E}[r_{S_t,a_t}]}{\sum_{t=1}^T\mathbb{E}[\check{c}_{S_t,a_t}]\sum_{t=1}^T\mathbb{E}[c_{S_t,a_t}]}\left\vert\sum_{t=1}^T\mathbb{E}[c_{S_t,a_t}]-\sum_{t=1}^T\mathbb{E}[\check{c}_{S_t,a_t}]\right\vert\\
    \leq& \frac{1}{Tc_{\min}}\vert\sum_{t=1}^T\mathbb{E}[\hat{r}_{S_t,a_t}]-\sum_{t=1}^T\mathbb{E}[r_{S_t,a_t}]\vert+ \frac{r_{\max}}{Tc_{\min}^2}\vert\sum_{t=1}^T\mathbb{E}[c_{S_t,a_t}]-\sum_{t=1}^T\mathbb{E}[\check{c}_{S_t,a_t}]\vert,
\end{align*}
where the last inequality comes from the boundedness of cost and reward.
This inequality breaks down the cumulative ratio estimate error into the estimated errors of reward and cost.
Consequently, we only need to bound the following terms:
\begin{equation*}
    \left\vert\sum_{t=1}^T\mathbb{E}[\hat{r}_{S_t,a_t}]-\sum_{t=1}^T\mathbb{E}[r_{S_t,a_t}]\right\vert,~\left\vert\sum_{t=1}^T\mathbb{E}[c_{S_t,a_t}]-\sum_{t=1}^T\mathbb{E}[\check{c}_{S_t,a_t}]\right\vert,
\end{equation*}
which have been widely studied in UCB/LCB learning~\cite{TorCsa_20,LiuLiBShi_21}.

We first prove the error bound for reward function $r$ and the analysis of cost function $c$ follows the same steps.
Define $x_{S_t, a}$ as the probability that arm $a\in \mathcal{A}(S_t)$ is pulled at time frame $t$ and we have
\begin{align*}
    \left\vert\sum_{t=1}^T\mathbb{E}[\hat{r}_{S_t,a_t}]-\sum_{t=1}^T\mathbb{E}[r_{S_t,a_t}]\right\vert
    = &\sum_{t=1}^T \sum_{a\in \mathcal{A}(S_t)}\mathbb{E}\left[(\hat{r}_{S_t,a} - r_{S_t,a})x_{S_t,a}\right]\\
    = &\sum_{a\in \mathcal{A}(S_t)}\left[\sum_{t=1}^T\mathbb{E}\left[(\hat{r}_{S_t,a} - r_{S_t,a})\mathbb{I}(\hat{r}_{S_t,a} - r_{S_t,a} \leq \zeta)x_{S_t,a}\right] + \mathbb{E}\left[(\hat{r}_{S_t,a}-r_{S_t,a})\mathbb{I}(\hat{r}_{S_t,a}-r_{S_t,a} > \zeta)x_{S_t,a}\right]\right]\\
    \leq& T\zeta\sum_{a \in \mathcal{A}(S_t)}\mathbb{E}[x_{S_t,a}] + \sum_{a \in \mathcal{A}(S_t)}\sum_{t=1}^T\mathbb{E}\left[(\hat{r}_{S_t,a}-r_{S_t,a})\mathbb{I}(\hat{r}_{S_t,a}-r_{S_t,a} > \zeta)x_{S_t,a}\right]\\
    \overset{(a)}{\leq}& T\zeta + \sum_{a \in \mathcal{A}(S_t)}\int_\zeta^\infty\sum_{t=1}^T\mathbb{P}(\hat{r}_{S_t,a} - r_{S_t,a} > y)dy\\
    \overset{(b)}{=} &T\zeta + \sum_{a \in \mathcal{A}(S_t)}\int_\zeta^{r_{\max}}\sum_{t=1}^T\mathbb{P}(\hat{r}_{S_t,a} - r_{S_t,a}> y) dy\\
    \overset{(c)}{\leq}& T\zeta + \vert \mathcal{A}(S_t)\vert\int_\zeta^{r_{\max}}\sum_{s=1}^T\mathbb{P}\left(\bar{R}_{S_t,a,s} + \sqrt{\frac{\log T}{s}} - r_{S_t,a} > y\right)dy ,
\end{align*}
where $(a)$ holds due to the definition of expectation;
$(b)$ holds because $\hat{r}_{S_t, i} - r_{S_t, i} \leq r_{\max}$;
$(c)$ holds due to the definition of $\hat{r}_{S_t,a}$.

Let $s(y) = \frac{\log T}{(1-\frac{1}{r_{\max}})^2y^2}$, we have
\begin{align*}
    \int_\zeta^{r_{\max}}\sum_{s=1}^T\mathbb{P}\left(\bar{R}_{S_t,i,s} + \sqrt{\frac{\log T}{s}}-r_{S_t,i}>y\right)dy\notag
    &\overset{(a)}{\leq} \int_\zeta^\infty s(y)dy + \int_\zeta^{r_{\max}}\sum_{s = \lceil s(y)\rceil}^T\mathbb{P}\left(\bar{R}_{S_t,i,s} - r_{S_t,i} > \frac{1}{r_{\max}}y\right)dy\notag\\
    &= \frac{\log T}{\left(1-\frac{1}{r_{\max}}\right)^2\zeta} + \int_\zeta^{r_{\max}}\sum_{s = \lceil s(y)\rceil}^T\mathbb{P}\left(\bar{R}_{S_t,i,s} - r_{S_t,i} > \frac{1}{r_{\max}}y\right)dy\notag\\
    &\overset{(b)}{\leq} \frac{\log T}{\left(1-\frac{1}{c_{\max}}\right)^2\zeta} + \int_\zeta^{r_{\max}}\sum_{s = \lceil s(y)\rceil}^T e^{\frac{-sy^2}{2r_{\max}^2}} dy\notag\\
    &\leq \frac{\log T}{\left(1-\frac{1}{r_{\max}}\right)^2\zeta} + \int_\zeta^{r_{\max}}\frac{e^{-\frac{\log T}{2(r_{\max}-1)^2}}}{1 - e^{-\frac{y^2}{2r_{\max}^2}}}dy\notag\\
    &\leq \frac{\log T}{\left(1-\frac{1}{r_{\max}}\right)^2\zeta} + e^{-\frac{\log T}{2(r_{\max}-1)^2}}\left(\int_\zeta^{\frac{r_{\max}}{2}} \frac{8}{4\frac{y^2}{r_{\max}^2} - \frac{y^4}{r_{\max}^4}}dy + \int_{\frac{r_{\max}}{2}}^{r_{\max}}\frac{1}{1-e^{-\frac{y^2}{2r_{\max}}}}dy\right)\notag\\
    &\leq \frac{\log T}{\left(1-\frac{1}{r_{\max}}\right)^2\zeta} + e^{-\frac{\log T}{2(r_{\max}-1)^2}}\left(\int_\zeta^{\frac{r_{\max}}{2}}\frac{8r_{\max}^2}{3y^2}dy + \frac{r_{\max}}{2(1-e^{-\frac{1}{2}})}\right)\notag\\
    &\leq \frac{\log T}{\left(1-\frac{1}{r_{\max}}\right)^2\zeta} + e^{-\frac{\log T}{2(r_{\max}-1)^2}}\left(\frac{8r_{\max}^2}{3\zeta} - \frac{16r_{\max}}{3} + \frac{r_{\max}}{2(1-e^{-\frac{1}{2}})}\right)\notag\\
    &\leq \frac{\log T}{\left(1-\frac{1}{r_{\max}}\right)^2\zeta} + e^{-\frac{\log T}{2(r_{\max}-1)^2}}\frac{8r_{\max}^2}{3\zeta}\notag\\
    &\leq r_{\max}^2\left(\frac{1}{(r_{\max}-1)^2} + \frac{8}{3}\right)\frac{\log T}{\zeta},
\end{align*}
where $(a)$ holds because $s(y)>0$;
$(b)$ holds due to Hoeffding's inequality.

Let $P = r_{\max}^2\left(\frac{1}{(r_{\max}-1)^2} + \frac{8}{3}\right)$ and $\zeta = \sqrt{\frac{\vert \mathcal{A}(S_t)\vert P\log T}{T}}$, we have
\begin{equation*}
    \left\vert\sum_{t=1}^T\mathbb{E}[\hat{r}_{S_t,a_t}]-\sum_{t=1}^T\mathbb{E}[r_{S_t,a_t}]\right\vert \leq \sqrt{\vert \mathcal{A}(S_t)\vert PT\log T},
\end{equation*}
which finishes the proof.

\section{Proof of Lemma~\ref{lem:sqrtsumbound}} \label{app:sum of sqrt Q_t^2}

Take summation over $T$ on equation in Lemma~\ref{lem:drift}, we have
\begin{equation*}
    \sum_{t=1}^T \mathbb{E}[\Delta(t)] \leq -2c_{\min}\sum_{t=1}^T\eta \mathbb{E}[V_t] + b\sum_{t=1}^T\eta ^2+ \theta_{gap}\sum_{t=1}^T\eta \mathbb{E}\left[\hat{r}_{S_t, a_t} - r_{S_t,a_t} - \theta^*(\check{c}_{S_t,a_t} - c_{S_t,a_t})\right].
\end{equation*}

Rearrange the inequality and divide both sides by $2c_{min} \eta$, we have
\begin{align*}
    \sum_{t=1}^T\mathbb{E}[V_t] \leq& -\sum_{t=1}^T \frac{\mathbb{E}[\Delta(t)]}{2c_{\min}\eta } + \sum_{t=1}^T \frac{\eta  b}{2c_{min}} +   \frac{\theta_{gap}}{2c_{min}}\sum_{t=1}^T \mathbb{E}\left[\hat{r}_{S_t,a_t} - r_{S_t,a_t}-\theta^*(\check{c}_{S_t,a_t} - c_{S_t,a_t})\right] \\
    \overset{(a)}{\leq} & \frac{ \mathbb{E}[V_1]}{2c_{min}\eta} + \sum_{t=1}^T \frac{\eta  b}{2c_{min}} + \frac{\theta_{gap}}{2c_{min}} \sum_{t=1}^T  \mathbb{E}\left[\hat{r}_{S_t,a_t} - r_{S_t,a_t}-\theta^*(\check{c}_{S_t,a_t} - c_{S_t,a_t})\right] \\
    \overset{(b)}{=} & \frac{\mathbb{E}[V_1]}{2}\sqrt{T} + \frac{b}{2c_{\min}^2}\sqrt{T}+ \frac{\theta_{gap}}{2c_{\min}}\sum_{t=1}^T\mathbb{E}\left[\hat{r}_{S_t, a_t} - r_{S_t,a_t} - \theta^*(\check{c}_{S_t,a_t} - c_{S_t,a_t})\right] \\
    \overset{(c)}{\leq} & \frac{\theta_{gap}^2}{2}\sqrt{T} + \frac{b}{2c_{\min}^2}\sqrt{T}+ \frac{\theta_{gap}}{2c_{\min}}\sum_{t=1}^T\mathbb{E}\left[\hat{r}_{S_t, a_t} - r_{S_t,a_t} - \theta^*(\check{c}_{S_t,a_t} - c_{S_t,a_t})\right],
\end{align*}
where the $(a)$ arises from $V_{T+1} \geq 0$;
$(b)$ is derived from $\eta  = \frac{1}{c_{\min}\sqrt{T}}$;
$(c)$ is due to the boundedness of the ratio $\theta$.
Recall that we have
\begin{equation*}
    \hat{r}_{S_t,a_t} = \bar{R}_{S_t,a_t} + \sqrt{\frac{\log T}{N_{S_t, a_t}(t)}}, ~\check{c}_{S_t,a_t} = \bar{C}_{S_t,a_t} - \sqrt{\frac{\log T}{N_{S_t, a_t}(t)}},
\end{equation*}
let $M = \sum^S \vert\mathcal{A}(i)\vert$ denote the number of all possible combinations $(S,a)$, we have\footnote{We assume all possible combinations $(S,a)$ are explored once before the decision making, which indicates that $\mathbb{E} [r_{S_t,a_t} - \bar{R}_{S_t,a_t}] = \mathbb{E} [c_{S_t,a_t} - \bar{C}_{S_t,a_t}] = 0,~\forall t \in [T].$}
\begin{align}
    \sum_{t=1}^T\mathbb{E}\left[\hat{r}_{S_t, a_t} - r_{S_t,a_t} - \theta^*(\check{c}_{S_t,a_t} - c_{S_t,a_t})\right] \leq& (1 + \theta_{\max})\sum_{t=1}^T\mathbb{E}\left[\sqrt{\frac{\log T}{N_{S_t, a_{t}}(t)}}\right]\nonumber\\
    =& (1 + \theta_{\max})\sqrt{\log T}\sum_{i = 1}^{M}\sum_{j = 1}^{N_i(T)}\frac{1}{\sqrt{j}}\nonumber\\
    \leq& (1 + \theta_{\max})\sqrt{\log T}\sum_{i = 1}^{M}\left(1+\int_1^{N_i(T)}\frac{1}{\sqrt{x}}dx\right)\nonumber\\
    \leq& 2(1 + \theta_{\max})\sqrt{\log T}\sum_{i = 1}^{M}\sqrt{N_i(T)}\nonumber\\
    \leq& 2(1 + \theta_{\max})\sqrt{T\log T}, \label{eq:Q bound}
\end{align}

Substitute the above inequality in~\eqref{eq:Q bound}, we have
\begin{equation*}
    \sum_{t=1}^T\mathbb{E}[V_t] \leq \frac{\theta_{gap}^2}{2}\sqrt{T} + \frac{b}{2c_{\min}^2}\sqrt{T}+\frac{(1 + \theta_{\max})\theta_{gap}\sqrt{T\log T}}{c_{\min}}.
\end{equation*}
According to Cauchy–Schwarz inequality, we have
\begin{equation*}
    (\sum_{t=1}^T\sqrt{\mathbb{E}[V_t]})^2\leq T\sum_{t=1}^T\mathbb{E}[V_t] \leq \frac{\theta_{gap}^2}{2}T^{\frac{3}{2}} + \frac{b}{2c_{\min}^2}T^{\frac{3}{2}}+\frac{(1 + \theta_{\max})\theta_{gap}\sqrt{\log T}}{c_{\min}}T^{\frac{3}{2}} = C_2 T^{\frac{3}{2}},
\end{equation*}
where $C_2 = \frac{\theta_{gap}^2}{2}+\frac{b}{2c_{\min}^2}+\frac{(1+\theta_{\max})\theta_{gap}\sqrt{\log T}}{c_{\min}}$.

Apply square roots on both sides of the above inequality, we finally have
\begin{equation*}
    \sum_{t=1}^T\sqrt{V_t} \leq \sqrt{C_2}T^{\frac{3}{4}}.
\end{equation*}

\newpage
\section{Additional Details on the Experiments}
\label{app:additional experiment details}
The following five tables are the net structures of the five types of tasks in real data experiments, which are written in the language of TensorFlow.
  \begin{table}[htpb]
   \caption{Model Setup of Satellite Image Classification Task.}
   \label{tab:satellite}
   \begin{tabular}{lll}
     \multicolumn{3}{l}{Model: "sequential"}\\
     \toprule
     \textit{Layer} & \textit{Output Shape} & \textit{Param \#} \\ \midrule
     conv2d & (None, 253, 253, 32) & 896   \\
     conv2d 1 & (None, 251, 251, 32) & 9248 \\
     max pooling2d & (None, 125, 125, 32) & 0 \\
     conv2d 2 & (None, 123, 123, 64) & 18496 \\
     max pooling2d 1 & (None, 61, 61, 64) & 0 \\
     conv2d 3 & (None, 59, 59, 128) & 73856 \\
     max pooling2d 2 & (None, 29, 29, 128) & 0 \\
     flatten & (None, 107648) & 0 \\
     dense & (None, 128) & 13779072 \\
     dropout & (None, 128) & 0 \\
     dense 1 & (None, 4) & 516 \\
     \toprule
     \multicolumn{3}{l}{Total params: 13,882,084}\\
     \multicolumn{3}{l}{Trainable params: 13,882,084}\\
     \multicolumn{3}{l}{Non-trainable params: 0}
   \end{tabular}
 \end{table}
 \begin{table}[htpb]
   \caption{Model Setup of Weather Classification Task.}
   \label{tab:weather}
   \begin{tabular}{lll}
     \multicolumn{3}{l}{Model: "sequential"}\\
     \toprule
     \textit{Layer} & \textit{Output Shape} & \textit{Param \#} \\ \midrule
     vgg19 & (None, 4, 4, 512) & 20024384   \\
     flatten & (None, 8192) & 0 \\
     dense & (None, 5) & 40965 \\
     \toprule
     \multicolumn{3}{l}{Total params: 20,065,349}\\
     \multicolumn{3}{l}{Trainable params: 40,965}\\
     \multicolumn{3}{l}{Non-trainable params: 20,024,384}
   \end{tabular}
 \end{table}
  \begin{table}[htpb]
   \caption{Model Setup of Rice Classification Task.}
   \label{tab:rice}
   \begin{tabular}{lll}
     \multicolumn{3}{l}{Model: "sequential"}\\
     \toprule
     \textit{Layer} & \textit{Output Shape} & \textit{Param \#} \\ \midrule
     conv2d & (None, 222, 222, 32) & 896   \\
     max pooling2d & (None, 111, 111, 32) & 0 \\
     flatten & (None, 394272) & 0 \\
     dense & (None, 40) & 15770920 \\
     dropout & (None, 40) & 0 \\
     dense 1 & (None, 5) & 205 \\
     \toprule
     \multicolumn{3}{l}{Total params: 15,772,021}\\
     \multicolumn{3}{l}{Trainable params: 15,772,021}\\
     \multicolumn{3}{l}{Non-trainable params: 0}
   \end{tabular}
 \end{table}
  \begin{table}[htpb]
   \caption{Model Setup of Nature Scene Classification Task.}
   \label{tab:nature scene}
   \begin{tabular}{lll}
     \multicolumn{3}{l}{Model: "sequential"}\\
     \toprule
     \textit{Layer} & \textit{Output Shape} & \textit{Param \#} \\ \midrule
     conv2d & (None, 148, 148, 32) & 896   \\
     max pooling2d & (None, 74, 74, 32) & 0 \\
     conv2d 1 & (None, 72, 72, 32) & 9248 \\
     max pooling2d 1 & (None, 36, 36, 32) & 0 \\
     dense & (None, 128) & 5308544 \\
     dense 1 & (None, 6) & 774 \\
     \toprule
     \multicolumn{3}{l}{Total params: 5,319,462}\\
     \multicolumn{3}{l}{Trainable params: 5,319,462}\\
     \multicolumn{3}{l}{Non-trainable params: 0}
   \end{tabular}
 \end{table}
 \begin{table}[htpb]
   \caption{Model Setup of Cats \& Dogs Classification Task.}
   \label{tab:cat_vs_dog}
   \begin{tabular}{lll}
     \multicolumn{3}{l}{Model: "model"}\\
     \toprule
     \textit{Layer} & \textit{Output Shape} & \textit{Param \#} \\ \midrule
      input2 & [(None, 150, 150, 3)] & 0   \\
     conv2d & (None, 148, 148, 32) & 896 \\
     conv2d 1 & (None, 146, 146, 64) & 18496 \\
     max pooling2d & (None, 73, 73, 64) & 0 \\
     conv2d 2 & (None, 71, 71, 64) & 36928 \\
     conv2d 3 & (None, 69, 69, 128) & 73856 \\
     max pooling2d 1 & (None, 34, 34, 128) & 0 \\
     conv2d 4 & (None, 32, 32, 128) & 147584 \\
     conv2d 5 & (None, 30, 30, 256) & 295168 \\
     global average pooling2d & (None, 256) & 0 \\
     dense & (None, 1024) & 263168 \\
     dense 1 & (None, 1) & 1025 \\
     \toprule
     \multicolumn{3}{l}{Total params: 837,121}\\
     \multicolumn{3}{l}{Trainable params: 837,121}\\
     \multicolumn{3}{l}{Non-trainable params: 0}
   \end{tabular}
 \end{table}


\end{document}